\documentclass{article}

\usepackage{microtype}
\usepackage{graphicx}
\usepackage{subcaption}
\usepackage{booktabs} % for professional tables
\usepackage{lipsum}  
\usepackage{enumitem}
\usepackage{subcaption}
 \usepackage[accepted, nohyperref]{icml-arXiv}

\usepackage{hyperref}

\usepackage{amsmath}
\usepackage{amssymb}
\usepackage{mathtools}
\usepackage{amsthm}

\usepackage[capitalize,noabbrev]{cleveref}
\usepackage{mdframed}

\newcommand{\df}{\mathrm{df}}
\newcommand{\trace}{\operatorname{tr}}
\theoremstyle{plain}
\newtheorem{theorem}{Theorem}[section]
\newtheorem{proposition}[theorem]{Proposition}
\newtheorem{lemma}[theorem]{Lemma}
\newtheorem{corollary}[theorem]{Corollary}

\newtheorem{remark}[theorem]{Remark}

\usepackage[textsize=tiny]{todonotes}

\icmltitlerunning{Model Collapse Demystified: The Case of Regression}

\begin{document}

\twocolumn[
\icmltitle{Model Collapse Demystified: The Case of Regression}




\begin{icmlauthorlist}
\icmlauthor{Elvis Dohmatob}{fair}
\icmlauthor{Yunzhen Feng}{cds}
\icmlauthor{Julia Kempe}{cds,courant,fair}
%\icmlauthor{Firstname6 Lastname6}{sch,yyy,comp}
%\icmlauthor{Firstname7 Lastname7}{comp}
%\icmlauthor{Firstname8 Lastname8}{sch}
%\icmlauthor{Firstname8 Lastname8}{yyy,comp}
\end{icmlauthorlist}

\icmlaffiliation{cds}{Center for Data Science, New York University}
\icmlaffiliation{courant}{Courant Institute, New York University}
\icmlaffiliation{fair}{Meta FAIR}

%\icmlcorrespondingauthor{Elvis Dohmatob}{dohmatob@meta.com}
\icmlcorrespondingauthor{Elvis Dohmatob}{dohmatob@meta.com}

\icmlkeywords{Machine Learning, ICML}

\vskip 0.3in
]

\printAffiliationsAndNotice{} 

\setcounter{tocdepth}{0} %% uncomment line for non-verbose ToC

\begin{abstract}
In the era of proliferation of large language and image generation models, the phenomenon of "model collapse" refers to the situation whereby as a model is trained recursively on data generated from previous generations of itself over time, its performance degrades until the model eventually becomes completely useless, i.e the model collapses. In this work, we study this phenomenon in the setting of high-dimensional regression and obtain analytic formulae which quantitatively outline this phenomenon in a broad range of regimes. In the special case of polynomial decaying spectral and source conditions, we obtain modified scaling laws which exhibit new crossover phenomena from fast to slow rates. We also propose a simple strategy based on adaptive regularization to mitigate model collapse. Our theoretical results are validated with experiments.
\end{abstract}

\section{Introduction}
\label{sec:intro}
\textit{Model collapse} describes the situation where the performance of large language models (LLMs) or large image generators degrade as more and more AI-generated data becomes present in their training dataset  \citep{shumailov2023curse}. Indeed, in the early stages of the generative AI evolution (e.g the ChatGPT-xyz series of models), there is emerging evidence suggesting that retraining a generative AI model on its own outputs can lead to various anomalies in the model's later outputs. 

This phenomenon has been particularly observed in LLMs, where retraining on their generated content introduces irreparable defects, resulting in what is known as ``model collapse", the production of nonsensical or gibberish output \citep{shumailov2023curse,bohacek2023nepotistically}. Though  several recent works demonstrate facets of this phenomenon {\em empirically} in various settings \citep{Hataya_2023_ICCV, martínez2023combining,martínez2023understanding,bohacek2023nepotistically,briesch2023large,guo2023curious}, a theoretical understanding is still missing.

In this work, we initiate a theoretical study of model collapse in the setting of high-dimensional supervised-learning with kernel regression. Kernel methods are popular in machine learning because, despite their simplicity, they define a framework powerful enough to exhibit non-linear features, while remaining in the convex optimization domain. While popular in their own right, kernel methods have made a recent spectacular comeback as proxies for neural networks in different regimes \citep{Belkin18understand}, for instance in the infinite-width limit \citep{Neal1996PriorsFI,NIPS1996_InfiniteNN,NEURIPS2018_Jacot,ICLR18_LeeNTK}  or in the lazy regime of training \citep{chizat}. 
\citet{Caponnetto2007OptimalRF} characterize the power-law generalization error of regularized least-squares kernel algorithms, assuming a power-decay spectrum of the kernel (capacity) and of the coefficients of the target function (source). The role of optimization can also be taken into account in this setting \citep{nitanda2021optimal}. Source and capacity power decay capture properties of the data and the model that give rise to power law scaling of test error in terms of data set size and model capacity, as empirically observed e.g.~in \citet{kaplan2020scaling,hoffmann2022trainingChinchilla}.   More recently, scaling laws have been shown for kernel models under the Gaussian design, e.g. in \citet{Spigler_2020,cui2021generalization,Cui_2022} for regression and \citet{Cui_2023} for classification. 
\citet{RahimiRecht2008,Rudy2017RF, maloney2022solvable} study scaling laws for regression in the random feature (kernel) model. 
In the nonparametric literature, for example \citet{SchmidtHieber2017scaling} and \citet{suzuki2018adaptivity} derived the test error scaling of deep neural network in fitting certain target functions and \citet{BordelonCP20spectrum} analyzed spectral dependence. 
%\bbb{These works leverage a power-law behavior in feature distribution to theoretically characterize the neural scaling law effect, which describes how the decrease in test loss can be modeled as a linear function of both data and model size, incorporating logarithmic scaling.}

%\bbb{These works leverage a power-law behavior in feature distribution to theoretically characterize the neural scaling law effect, which describes how the decrease in test loss can be modeled as a linear function of both data and model size, incorporating logarithmic scaling.}

\begin{figure*}[!tb]
    \centering
    \includegraphics[width=0.8\textwidth]{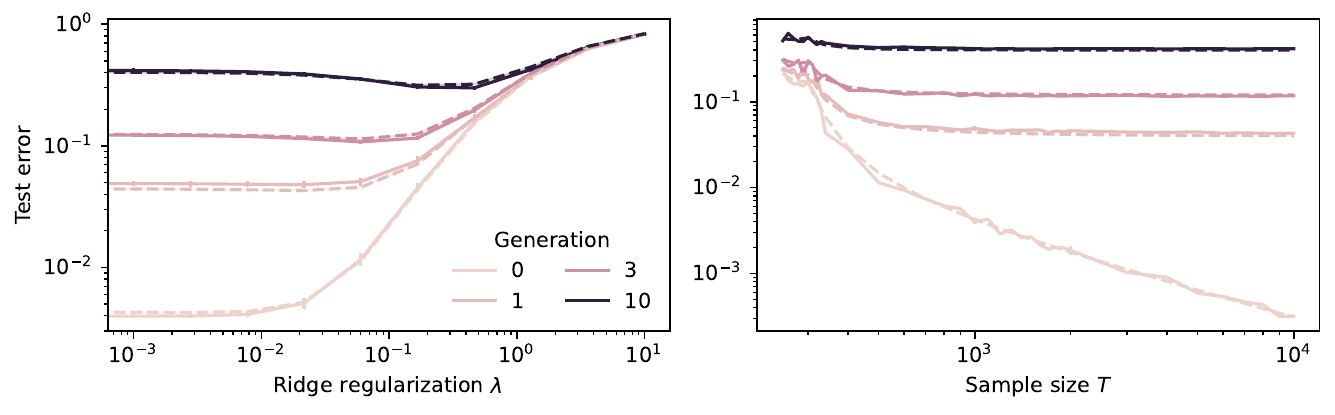}
    \vspace{-.25cm}
    \caption{\textbf{Demystifying model collapse in ridge regression (isotropic covariance spectrum).} We show the evolution of test error for different sample size ($T$), different levels of ridge-regularization ($\lambda$), and training data from different generations ($n$) of fake data. The setup is: input-dimension $d=300$, sample size for fake data generator $T_0=600$, noise levels $\sigma=0.1$ and $\sigma_0 = 0.2$. \textbf{Left plot} is for $T=1000$ and different values of $\lambda$. Notice the U-shape of the curves for large values of $n$, indicating the existence of a sweet spot (optimal regularization parameter).  \textbf{Right plot} is for $\lambda = 10^{-3}$ and different values of $T$. Error bars correspond to uncertainty induced by the data-generating process, over different runs. The broken lines correspond to the theoretical result established in Theorem \ref{thm:linreg}.}
    \label{fig:linreg}
    %\vspace{-0.4cm}
\end{figure*}

\begin{figure*}[!htb]
    \centering
    \includegraphics[width=0.8\textwidth]{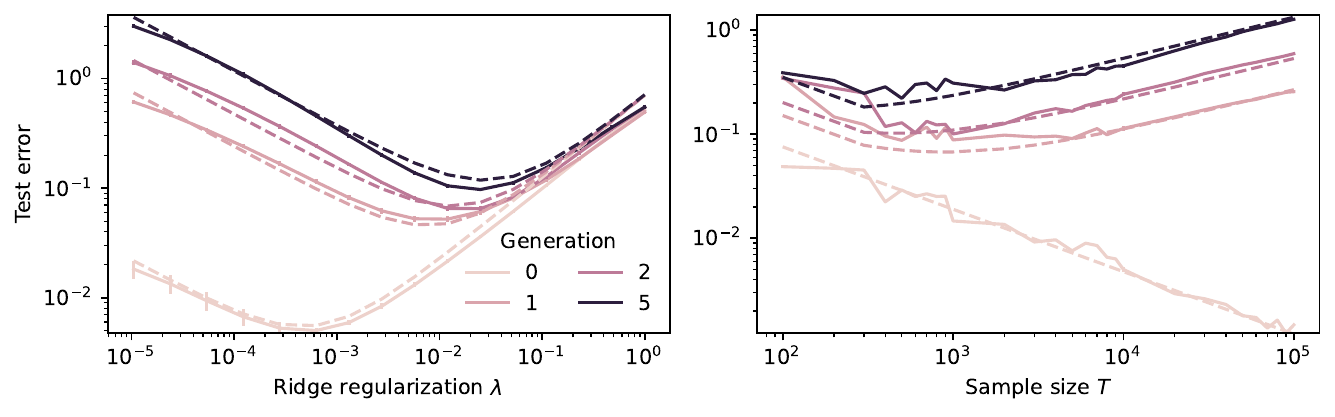}
    \vspace{-.25cm}
    \caption{\textbf{Demystifying model collapse in ridge regression (power-law covariance spectrum)}. The setup is: $d=300$, $T_0=600$, $\sigma=\sigma_0=1$, $\Sigma=\mathrm{diag}(\lambda_1,\ldots,\lambda_d)$, where $\lambda_k \propto k^{-2}$. \textbf{Left plot} corresponds to $T=10,000$ and \textbf{Right plot} corresponds to adaptive regularization $\lambda=T^{-\ell_{crit}}$, where $\lambda=\lambda(T)$ as proposed in \citet{Cui_2022}.  See Section \ref{sec:exp} for details. The broken curves are as predicted by our Theorem \ref{thm:darkseid}. Though $\ell=\ell_{crit}$ is optimal in classical case, it is not in the setup of model collapse. In fact here, the test error diverges with sample size $T$. Our theory proposes a corrected value of this exponent which gracefully adapts to synthesized data.}
    \label{fig:krreg}
    %\vspace{-0.25cm}
\end{figure*}

\begin{figure*}[!htb]
    \centering
    \begin{subfigure}{\textwidth}
        \centering
    \includegraphics[width=0.8\textwidth]{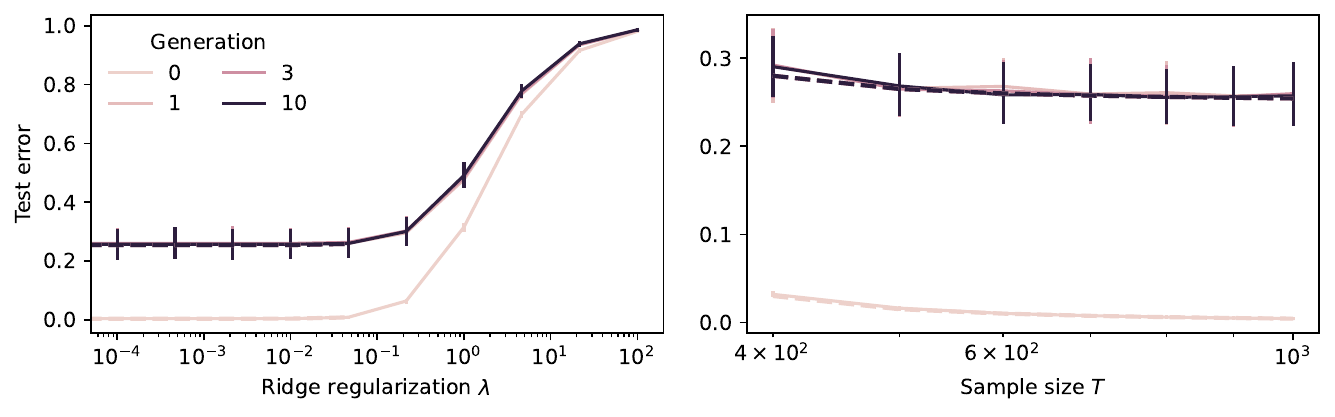}
    \caption{Identical intermediate design matrices $X_n=X_0$ for all $n \ge 1$.}
    \end{subfigure}
    \begin{subfigure}{\textwidth}
        \centering
    \includegraphics[width=0.8\textwidth]{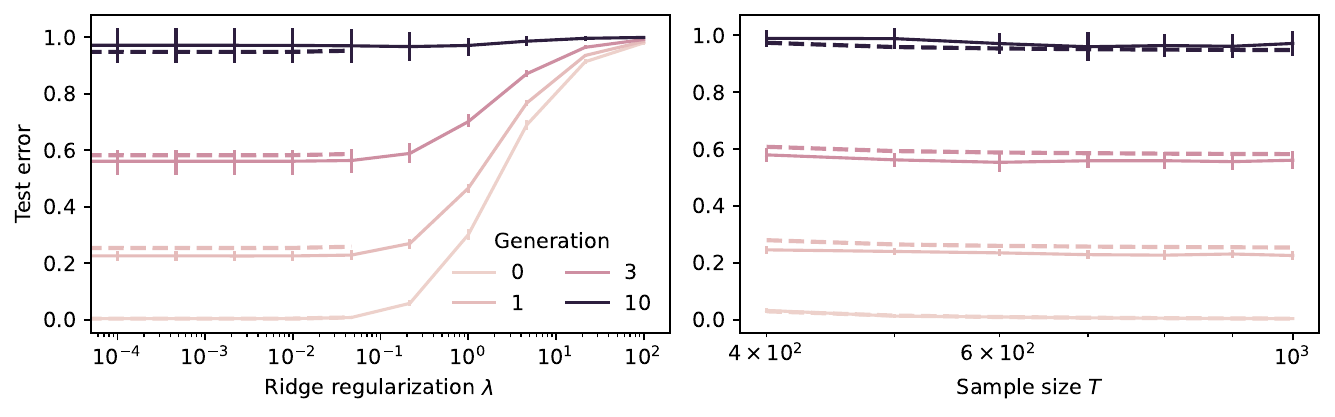}
    \caption{Independent intermediate design matrices $X_0,\ldots,X_n$.}
    \end{subfigure}
        \vspace{-.75cm}
    \caption{\textbf{Model collapse in the case of noiseless over-parametrized synthetic data generator.} Here $d=300$, the sample sizes for the different versions of the fake data generator are equal, i.e $T_n=T_0=d/2$ for all $n$, and noise levels are $\sigma_0=0$ and $\sigma=0.1$. Everything else is as in the setting of Figure \ref{fig:linreg}. Broken lines correspond to the theoretical estimates given in Theorem \ref{thm:rho}. As predicted by our theory, the test error of the model fitted on synthetic data ($n \ge 1$) increases (relative to the baseline $n=0$, corresponding to training on clean data). The model collapse here, even in the absence of noise ($\sigma_0=0$), is due to the fact that the synthetic data-generator does not have access to enough data to capture the true labelling function. \textbf{(a)} Importantly, and in accordance to our theory, the amount of model collapse in the case $X_n \equiv X_0$ is due to an increase in bias term of the test error of the model and does not depend on the number of generations $n$ as long as $n \ge 1$. \textbf{(b)} In contrast, for the case where the $X_n$'s are independent, the increase in bias term grows with $n$, leading to "catastrophic" model collapse (Theorem \ref{thm:cata}).}
    \label{fig:harmsway}
    \vspace{-0.5cm}
\end{figure*}

\paragraph{Summary of Main Contributions.} Following the rich tradition in prior works outlined above, we study the Gaussian design where the input $x$ is sampled from a multivariate zero-mean Gaussian and labels $y$ are determined by a linear ground truth function with independent label noise as $y=x^\top w_0+\epsilon$ (we present the linear regression setting for ease, the generalization to the kernel setting is straightforward). At each generation step, an approximation to $w_0$ is learned from the data, and used to generate new, ``fake"/synthetic labels for the next generation. 
Our main findings can be summarized as follows:

\textit{(1) Exact Characterization of Test Error under Iterative Retraining on Synthesized Data.} In Section \ref{sec:generalcov} (Theorem \ref{thm:generalcov}), we obtain analytic formulae for test error under the influence of training data with fake / synthesized labels. 
% \paragraph{Model Collapse over Multiple Generation.}
For $n$-fold iteration of data-generation, this formula writes
\begin{eqnarray}
\label{eq:cata}
E_{test} = E_{test}^{clean} + \textcolor{red}{\Delta Bias} + \textcolor{red}{n \times Scaling},
\end{eqnarray}
where $E_{test}^{clean}$ is the usual test error of the model trained on clean data (not AI-generated). The non-negative term \textcolor{red}{$Scaling$} precisely highlights the effects of all the relevant problems parameters like: the feature covariance matrix $\Sigma$, sample size $T$, label noise level in the clean data distribution $\sigma_0^2$, label noise level in the fake data distribution $\sigma^2$, etc. The nonnegative term \textcolor{red}{$\Delta Bias$}
is an increase in bias brought about by the iterative synthetic data generation process. This term disappears if each stage in the process was fitted on sufficiently many samples $T_0$ compared to the input dimension $d$ (i.e if $T_0 \ge d$). In the case where $T_0 < d$, this term is either a constant or an increasing function of $n$, depending on whether the design matrix stays the same or is resampled across different generations. Note that the machine learner has no control over the fake data generation process. It only sees data from a stage $n$ of this process, which is then used to fit a downstream predictor.

A direct consequence of \eqref{eq:cata} % of this successive degrations % multiplicative degradation
is that, as the number of generations $n$ becomes large, the effect of re-synthesizing will make learning impossible. We note that the multiplicative degradation in scaling with the number of generations $n$ is completely analogous to what has been shown in \cite{dohmatob2024tale} for infinite memory models and their variants and also empirically observed there. See Figures \ref{fig:linreg} and \ref{fig:harmsway} for an illustration.

%A direct consequence of this multiplicative degradation is that, over time (i.e as the number of generations becomes large), the effect of large language models (like ChatGPT) in the wild will be a pollution of the web to the extent that learning will be impossible. This will likely increase the value and cost of clean / non-AI-generated data. 

\textit{(2) Modified Scaling Laws.} 
Turning to the special case of power-law spectra of the covariance matrix, which allows to derive test-error scaling laws \citep{Caponnetto2007OptimalRF,Spigler_2020,Cui_2022,JustInterpolate}, we obtain in Section \ref{sec:power} (see Theorem \ref{thm:darkseid}) precise new scaling laws of the test error that quantitatively highlight the negative effect of training on synthetically generated data. 
%\bbb{To understand the phenomenon of model collapse within the scaling law regime, we study the power-law scaling of eigenspectra of the covariance matrix as in previous research \citep{Spigler_2020,Cui_2022,JustInterpolate}. In Section \ref{sec:power} (see Theorem \ref{thm:darkseid}), we introduce precise new scaling laws that quantitatively illuminate the negative impact of training with synthetic data.}
% \textit{(3) New Crossover Phenomena and Optimal Ridge.}

Further exploiting our analytic estimates, we obtain (Corollary \ref{cor:lopt}) the optimal ridge regularization parameter as a function of all the problem parameters (sample size, spectral exponents, strength of fake data-generator, etc.). This new regularization parameter corresponds to a correction of the the value proposed in the classical theory on clean data \citep{Cui_2022}, and highlights a novel crossover phenomenon where for an appropriate tuning of the regularization parameter, the effect of training on fake data is a degradation of the fast error rate in the noiseless regime \citep{Cui_2022,Caponnetto2007OptimalRF} 
to a much slower error rate which depends on the amount of true data on which the fake data-generator was trained in the first place. On the other hand, a choice of regularization which is optimal for the classical setting (training on real data), might lead to catastrophic failure: the test error diverges. See Figure \ref{fig:krreg} for an illustration.

% Taken together, our results suggest that the model collapse phenomenon \citet{shumailov2023curse} could simply be a change of a scaling law after all. \yunzhen{Why saying this?}

%\bbb{%Apart from the above contributions, we hope the arguments and techniques used to derive our results will find broader use in the community. 
% We have employed tools from Random Matrix Theory (RMT), pushing forward the envelope of prior work. We hope our technical contributions will find broader use in the community and lend themselves to generalization. }
%\ElvisIssue{Maybe reformulate.} \julia{Please do!}

%    \ElvisIssue{We should make it clear that we're not studying kernel methods for the sake of studying kernel methods, but as a proxy for study neural networks.} \julia{I don't know how much we want to melanger the power laws for kernels with just analysis of kernels. At the moment the reference does both.}

%\clearpage % XXX rm hack

\section{Review of Literature}
\label{sec:related}
\paragraph{Model Collapse.} Current LLMs \citep{devlin2018bert, liu2019roberta, NEURIPS2020_1457c0d6, touvron2023llama}, including GPT-4 \citep{achiam2023gpt}, were trained on predominantly human-generated text; similarly, diffusion models like DALL-E \citep{pmlr-v139-ramesh21a}, Stable
Diffusion \citep{Rombach_2022_CVPR}, Midjourney \citep{midjourney2023} are trained on web-scale image datasets. Their training corpora already potentially exhaust all the available clean data on the internet. A growing number of synthetic data generated with these increasingly popular models starts to populate the web, often indistinguishable from ''real" data. Recent works call attention to the potential dramatic deterioration in the resulting models, an effect referred to as {\em ''model collapse"} \citep{shumailov2023curse}. Several recent works demonstrate facets of this phenomenon {\em empirically} in various settings \citep{Hataya_2023_ICCV, martínez2023combining,martínez2023understanding,bohacek2023nepotistically,briesch2023large,guo2023curious}. Theoretically, a few works are emerging to analyze the effect of iterative training on self-generated (or mixed) data. \citet{Taoridatafeedback23} call attention to {\em bias} amplification of iterative ``data-feedback" loops, also observed in the recommendations systems literature. % where feedback loops create echo chambers, and can be viewed as one form of model collapse.
\citet{shumailov2023curse} attribute model collapse to two mechanisms: a finite sampling bias leading to more and more peaked distributions and function approximation errors, and analyze the (single) Gaussian case. In the context of vision models, \citet{alemohammad2023selfconsuming} analyze {\em ''self-consuming loops"} by introducing a sampling bias that narrows the variance of the data at each generation, and provide theoretical analysis for the case of a Gaussian.
\citet{bertrand2023stability} explore scenarios involving a mix of clean data, representative of the true distribution, and synthesized data from previous iterations of the generator. Their analysis reveals that if the data mix consists exclusively of synthesized data, the generative process is likely to degenerate over time, leading to what they describe as a 'clueless generator'. %Thus, the generator collapses: it progressively loses its ability to capture the essence of the data distribution it was intended to model. 
Conversely, they found that when the proportion of clean data in the mix is sufficiently high, the generator, under certain technical conditions, retains the capability to learn and accurately reflect the true data distribution. % This work sheds light on the critical role of data composition in the stability and effectiveness of generative models.
Note that such a compounding effect of synthesized data is already reminiscent of our decomposition \eqref{eq:cata}.
%We stress here that unlike \citet{bertrand2023stability}, we are interested here in supervised learning, where only the labels are synthesized.
% \yunzhen{Make it shorter? Are those gaussian paper that related? Add one sentences on, how our method is different and why those analysis are not enough.}

\paragraph{Self-Distillation.} Importantly, the fake data generation process which is responsible for model collapse should not be confused with self-distillation as formulated in \citet{bartlett2020neurips} for example. Unlike model collapse, the data generation process in self-distillation actually helps performance of the downstream model. The model has access to training labels from the true data distribution, but decides to fit a model on this data, and then use its outputs as the new labels, iterating this process possibly over severable steps. Thus, self-distillation has control over the data generating process, which is carefully optimized for the next stage training.  Specifically, \cite{bartlett2020neurips} study self-distillation in the same Gaussian regression model underlying our analysis, but in each distillation generation are able to tune the regularization parameter for downstream performance as a function of the original data labels (with the data being the same at each generation). In the setting of model collapse, there is no control over the data generation process, since it constitutes synthesized data which typically comes from the wide web.

\begin{figure*}[!th]
    \centering
    \includegraphics[width=1\linewidth]{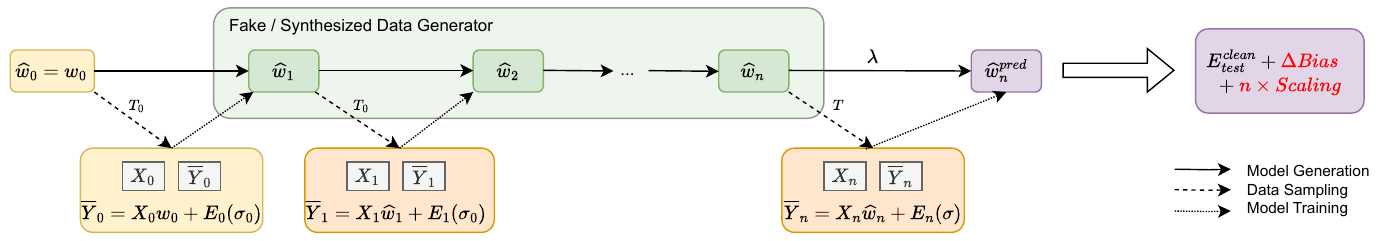}
%    \vspace{-10mm}
        \caption{Illustration of the theoretical framework. The process begins with the original model $\widehat{w}_0 (w_0)$ and the original dataset $(X_0, \overline{Y}_0)$. $n$ synthetic data generators $\widehat{w}_1$ to $\widehat{w}_n$ are iteratively fit on data labelled by the previous model with label noise $\sigma_0$, using $T_0$ samples each. We evaluate the test error of  $\widehat{w}_n^{pred}$ (with respect to the ground truth labels from $w_0$), which is trained on $(X,Y):=(X_n, \overline{Y}_n)$ using $T$ samples and a regularization coefficient $\lambda$.}
    \label{fig:teaser}
\end{figure*}
\paragraph{Kernel Ridge Regression with Gaussian Design.} This model has been studied by a vast body of works. For example, \citet{Richards2021,Hastie2019Surprises,bach2023highdimensional} analyze the classical bias-variance decomposition of the test error for ridge regression in the high dimensional setting where dataset size and dimension diverge proportionately, using tools from Random Matrix Theory (RMT). In Section \ref{sec:generalcov} we significantly extend this type of analysis to training on iteratively generated synthetic data. This model is also particularly attractive 
because it allows to analyze an important trade-off: the relative decay of the eigenvalues of the kernel ({\em capacity}) and the coefficients of the target function in feature space ({\em source}). Sizeable effort has been dedicated to characterize the influence on the decay rate of the test error as a function of these two relative decays (aka {\em power laws}) \citep{Caponnetto2007OptimalRF,PillaudVivien18,Berthier20,
Richards2021, Spigler_2020,Cui_2022,Cui_2023}. In Section \ref{sec:power} we extend these efforts, in particular based on works of \citet{cui2021generalization,Cui_2022} which has given a full characterization of all regimes and test error decay that can be observed at the interplay of noise and regularization, characterizing a crossover transition of rates in the noisy setting.   Our work uncovers fascinating new effects as a result of iterative training on synthetic data.

% \textcolor{blue}{\textit{\lipsum[12]}}

\section{Theoretical Setup}
\label{sec:setup}

We now present a setup which is simple enough to be analytically tractable, but rich enough to exhibit a wide range of regimes to illustrate a range of new phenomena that emerge with \textit{model collapse},  described in Section \ref{sec:intro} and Section \ref{sec:related}. Figure \ref{fig:teaser} provides a visual representation.

\textbf{Notations.}
This manuscript will make use of the following standard notations. The set of integers from $1$ through $d$ is denoted $[d]$. Given a variable $z$ (which can be the input dimension $d$ or the sample size $T$, etc.) the notation $f(z) \lesssim g(z)$ means that $f(z) \le C g(z)$ for sufficiently large $z$ and an absolute constant $C$, while $f(z) \asymp g(z)$ means $f(z) \lesssim g(z) \lesssim f(z)$. For example, $1+z^2 \asymp \max(1,z^2)$. Further, $f(z) \simeq g(z)$ means $f(z)=(1+o(1))g(z)$, where $o(1)$ stands for a quantity which tends to zero in the limit $z \to \infty$.
% $X^\dagger$ stands for the Moore-Penrose pseudo inverse of a matrix $X$.
We denote with $A^\dagger$ the Moore-Penrose pseudo-inverse any matrix $A$, and by $\|A\|_{op}$ is operator norm. Finally, we denote the trace of a square matrix $A$ by $\trace A$, while $\|u\|_A := \sqrt{u^\top A u}$ is the Mahalanobis norm induced by a positive-definite matrix $A$.

 %\bbb{% \paragraph{Parameters.} The following notation will be used regarding our problem setup. $d$ denotes the data dimension and $n$ the number of generations of synthetic data. $\Sigma$ and $\sigma$ pertain to the true data distribution given in (\ref{eq:regprob}. $T$ and $T_0$ denote dataset size, and, jointly with generation-dependent label noise $\sigma_m$, linear labellers $\hat{w}_m$ and regularization $\lambda_0$ are defined in (\ref{eq:gptchain}). }

\textbf{Data Distribution \& Fake Data-Generation Process.}
Consider the distribution $P_{\Sigma,w_0,\sigma^2}$ on $\mathbb R^d \times \mathbb R$ given by
\vspace{.2cm}
\begin{mdframed}
\bf{Data Distribution.}
\begin{eqnarray}
  \left.
    \begin{split}
    &\textbf{(Input) }x \sim N(0,\Sigma),\\
    &\textbf{(Noise) }\epsilon \sim N(0,\sigma^2),\text{ independent of }x\\
    &\textbf{(Output / Label) }y = x^\top  w_0 + \epsilon.
    \end{split}
\right\}
\label{eq:regprob}
\end{eqnarray}
\end{mdframed}
% \begin{mdframed}
%     \begin{eqnarray}
%     \begin{split}
%     &\textbf{(Input) }&x \sim N(0,\Sigma),\\
%     &\textbf{(Noise) }&\epsilon \sim N(0,\sigma^2),\text{ independent of }x\\
%     &\textbf{(Output / Label) }&y = x^\top  w_0 + \epsilon.
%     \end{split}
% \label{eq:regprob}
% \end{eqnarray}
% \end{mdframed}
% Thus, for any given input $x \in \mathbb R^d$, the target variable $y$ has Gaussian distribution with mean $x^\top w_0$ and variance $\sigma^2$.
The positive integer $d$ is the input-dimension, the vector $w_0 \in \mathbb R^d$ defines the ground-truth labelling function $x \mapsto x^\top w_0$, the matrix $\Sigma \in \mathbb R^{d \times d}$ is the covariance structure of the inputs. The scalar $\sigma^2$ is the level of label noise. To begin,  we consider the linear case for clarity. We shall discuss the kernel setting at the end of this section. Thus, in classical linear regression, we are given a sample $(X,Y) \equiv \{(x_1,y_1),\ldots,(x_T,y_T)\}$ of size $T$ from $P_{\Sigma,w_0,\sigma^2}$ and we seek a linear model $\widehat w \in \mathbb R^d$ with small test error % $E_{test}(\widehat w)$, defined by
\begin{eqnarray}
\label{eq:testerr}
\begin{aligned}
E_{test}(\widehat w) &:= \mathbb E_{\widehat w}\,\mathbb E_{x,y}[(x^\top \widehat w - y)^2] - \sigma^2\\
&= \mathbb E_{\widehat w}\,[\|\widehat w - w_0\|_\Sigma^2],
\end{aligned}
\end{eqnarray}
where $(x,y) \sim P_{\Sigma,w_0,\sigma^2}$ is a random clean test point.

% \textcolor{blue}{\textit{\lipsum[36]}}

In our setup for studying model collapse, the training data $(X,Y)$ is sampled from an iterative loop
% the design matrix $X \in \mathbb R^{T \times d}$ stays the same \julia{Is it? Or does it come from the *same* distribution?}, but the vector of labels $Y \in \mathbb R^T$ is generated by an iterative relabelling process,
where each generation of the model serves as the labeller for the data for the next generation.
This process is described below.

%\newpage

\begin{mdframed}
\textbf{Construction of Fake / Synthesized Data Generator.}
\begin{eqnarray}
\label{eq:gptchain}
    P_{\Sigma, w_0,\sigma_0^2} \to P_{\Sigma,\widehat w_1,\sigma_1^2} \to \ldots  \to  P_{\Sigma,\widehat w_n,\sigma_n^2 \to \ldots},
\end{eqnarray}
where the sequence of models $(\widehat w_n)_{n \in \mathbb N}$ is defined recursively by 
\begin{eqnarray}
\widehat w_n = \begin{cases}
w_0,&\mbox{ if }n = 0,\\
\operatorname{Fit}(X_{n-1}, \overline Y_{n-1}),&\mbox{ if }n \ge 1,
\end{cases}
\label{eq:fakew}
\end{eqnarray}
where $\overline Y_n := X_n\widehat w_n + E_{n}$ %\widehat Y_n(X_n)$, $\widehat Y_n(A) := A \widehat w_n + E_{n}$ 
and $\operatorname{Fit}(A,B) = \operatorname{OLS}(A,B) := A^\dagger B$ is  ordinary-least squares (OLS).
% In regression setting considered here,

The design matrices $(X_n)_{n \ge 0}$ are of shapes $T_n \times d$, each with iid rows from $N(0,\Sigma)$.

The sequence of noise vectors $(E_n)_{n \ge 0}$ forms an independent collection, which is independent of the $(X_n)_{n \ge 0}$ ; each $E_n=(\epsilon_{n,1},\ldots,\epsilon_{n,T_n})$ has length $T_n$ and iid components from $N(0,\sigma_n^2)$. %Each $X_n$ is of shape $\mathbb R^{T_n \times d}$ with iid rows from $N(0,\Sigma)$.*
This is summarized in Figure \ref{fig:teaser}.
\end{mdframed}
Thus, in summary, each $\widehat w_n$ results from fitting a model on a dataset of size $T_{n-1}$ from $P_{\Sigma,\widehat w_{n-1},\sigma_{n-1}^2}$, for every $n \ge 1$.
Importantly, the downstream model (introduced later) has no control over this process. It will only see training data from a given version $P_{\Sigma,\widehat w_n,\sigma_n^2}$, but will be tested on the true distribution $P_{\Sigma,w_0,\sigma_0^2}$.

\textbf{The Downstream Model: Ridge Regression.}
For a number of iterations $n \ge 0$, noise levels $\sigma_0$ and $\sigma$, dataset sizes $T_0$ and $T$, and regularization parameter $\lambda \ge 0$, let $\widehat w_n^{pred} = \widehat w^{pred}_{n,T_0,\sigma_0^2,T,\sigma,\lambda} \in \mathbb R^d$ be the ridge predictor constructed from an iid sample $\{(x_1,y_1),\ldots,(x_T,y_T)\}$ of size $T$ from the $n$-fold fake data distribution $P_{\Sigma,\widehat w_n,\sigma_n^2}$, where for ease of presentation of our results we will assume that
%\begin{mdframed}
    \begin{eqnarray}
    \left.
    \begin{split}
T_{n-1}&=\ldots=T_1=T_0, \quad T_n = T \\
\sigma_{n-1}&=\ldots=\sigma_1=\sigma_0,\quad \sigma_n = \sigma.
\end{split}
\right\}
\label{eq:identification}
\end{eqnarray}
%\end{mdframed}
 %$P_{\Sigma,\widehat w_n,\sigma^2}$.
%The aforementioned predictor is given by
 \begin{mdframed}
 {\bf Downstream Ridge Predictor.}
 
 For an $n$-fold fake data generator $P_{\Sigma,\widehat w_n,\sigma_n^2}$, we denote with $X:=X_n \in \mathbb R^{T \times d}$ the design matrix with iid rows from $N(0,\Sigma)$, with $E:=E_n \in \mathbb R^T$ the error vector with components in $N(0,\sigma_n^2)$, and $Y:=\overline Y_n =X\widehat w_n+E \in \mathbb R^T$ the labels generated by $P_{\Sigma,\widehat w_n,\sigma_n^2}$. Then
    \begin{eqnarray}
\label{eq:ridge}
    \widehat w_n^{pred} = \begin{cases}
    X^\dagger Y,&\mbox{ if }\lambda=0,\\
    R X^\top Y/T,
    &\mbox{ otherwise,}
    \end{cases}
\end{eqnarray}
where $$\widehat \Sigma := X^\top X /T \in \mathbb R^{d \times d}$$ is the sample covariance matrix, and $$R=R(\lambda) := (\widehat \Sigma + \lambda  I_d)^{-1}$$ denotes its resolvent.
\end{mdframed}

We are interested in the dynamics of the test error $E_{test}(\widehat w^{pred}_n)$ (according to formula \eqref{eq:testerr}) of this linear model. Importantly, the evaluation of the model is performed on the true data distribution $P_{\Sigma,w_0,\sigma_0^2}$, even though the model is trained on the fake data distribution $P_{\Sigma,\widehat w_n,\sigma^2}$. Note that for $n=0$, $E_{test}^{clean} := E_{test}(\widehat w_n^{pred})$ corresponds to the usual test error when the downstream model is trained on clean data.

%\yunzhen{I think here we have a problem on the subscript. Here, we are using $\widehat w_0^{pred}$ to denote the first downstream model trained on clean data. However, in the box of 'construction of Fake/synthesized data generator', we are using $\widehat w_1$ to denote the first distribution, it should be $\widehat w_0$. In Assumption 4.1, we are using $X_0$ to denote the original data? We could change Eqn(4) and use $X$ for the original data.} \julia{I need to think. You seem to say that instead of $\widehat w_0^{pred}$ we should call this $\widehat w_1^{pred}$ to have all indices aligned. But that gives other problems, I think. }

  The mental picture is as follows: each generation $\widehat w_n$ %($m \in \{1,2,\ldots,n\}$) 
  can be seen as a proxy for a specific version of ChatGPT, for example.
The sample size $T_0$ used to create the fake labelling functions $\widehat w_n$ is a proxy for the strength of the fake data-generator thus constructed. % Think of successive versions of ChatGPT. 
 Other works which have considered model collapse under such a self-looping training process include \citet{shumailov2023curse,alemohammad2023selfconsuming,bertrand2023stability,dohmatob2024tale}.

% To avoid pathological cases, we will always assume 
%\begin{eqnarray}
%T_0, T \ge d + 2.
%\end{eqnarray}

% This ensures that all the $\widehat w_n$'s and $\widehat w_n^{pred}$'s can be obtained via least squares.

\paragraph{A Note on Extension to Kernel Methods}
% \textcolor{red}{Reformulate as in \citet{SimonNTK2021}!}
Though we present our results in the case of linear regression in $\mathbb R^d$ for clarity, they can be rewritten in equivalent form in the kernel setting. Indeed, as in \citet{Caponnetto2007OptimalRF,SimonNTK2021,Cui_2022,JustInterpolate}, it suffices to replace $x$ with a feature map induced by a kernel $K$, namely $\psi(x) := K_x \in \mathcal H_K$.
% which is Gaussian (possibly infinite-dimensional) with covariance operator $\Sigma$.
Here, $\mathcal H_K$ is the reproducing kernel Hilbert space (RKHS) induced by $K$.
In the data distribution \eqref{eq:regprob}, we must now replace the Gaussian marginal distribution condition $x \sim N(0,\Sigma)$ with $\psi(x) \sim N(0,\Sigma)$. The ground-truth labeling linear function in \eqref{eq:regprob} is now just a general function $f_0 \in L^2$.
% , where $\mathcal H$ is a Hilbert space which now plays the role of $\mathbb R^d$, and into which the induced reproducing kernel Hilbert space (RKHS) $\mathcal H_K$ embeds continuously.
The ridge predictor \eqref{eq:ridge} is then given  by (\textit{Representer Theorem}) $\widehat f_n^{pred}(x) := K(X,x)^\top \widehat c_n$, with $\widehat c_n = (G + \lambda T I_d)^{-1}Y \in \mathbb R^n$,
where $K(X) := (K_{x_1},\ldots,K_{x_T})$, %$:\mathcal H_K \to \mathbb R^n
and $G = K(X)K(X)^\top = K(X,X) = (K(x_i,x_j))_{1 \le i,j\le T} \in \mathbb R^{n \times n}$ is the Gram matrix.
% Equivalently, we would replace $\widehat \Sigma$ with the empirical kernel integral-operator $K(X)^\top K(X)/T$ and $X^\top Y/T$ by $K(X)^\top Y/T$ in \eqref{eq:ridge}.

\section{Exact Test Error Characterization}
\label{sec:generalcov}
In this section we establish generic analytic formulae for the test error of the downstream model $\widehat w_n^{pred}$ \eqref{eq:ridge} trained on $n$-fold fake data-generation as outlined in Section \ref{sec:setup}. All proofs of this section are relegated to Appendix \ref{app:A}.

%\begin{assumption}
%In all our results, when $T_0 < d + 2$, we will assume for ease of presentation that in the fake data-generation process (\ref{eq:gptchain}), $X_n=X_0$ for all $n \ge 0$ \bbb{(but we impose no such equality on the final test distribution $X$ in (\ref{eq:ridge}))}. When $T_0 \ge d+2$, we do not make such an assumption.
%\end{assumption}

% We now consider general covariance matrix $\Sigma$ for the features.

\subsection{Warm-up: Unregularized Case}
For a start, let us first consider the case of ridgeless regression (corresponding to $\lambda=0$ in Equation \eqref{eq:ridge}), which amounts to OLS.
\begin{theorem}
For an $n$-fold fake data generation process with $T_0 \ge d+2$ samples, the test error for the linear predictor $\widehat w_n^{pred}$ given in \eqref{eq:ridge} learned on $T \ge d+2$ samples, with $\lambda=0$ (i.e unregularized), is given by
\begin{eqnarray}
E_{test}(\widehat w_n^{pred}) \simeq \frac{\sigma^2\phi}{1-\phi} + \textcolor{red}{\frac{n\sigma_0^2\phi_0}{1-\phi_0}},\label{eq:unreg}\\ \text{ with }\phi=\frac{d}{T}, \, \phi_0=\frac{d}{T_0}. \nonumber 
\end{eqnarray}
\vspace{-0.5cm}
% where $\phi := d/T$ and $\phi_0 := d/T_0$.
\label{thm:linreg}
\end{theorem}
The first term  $E_{test}(\widehat w_0^{pred})\simeq \sigma^2\phi/(1-\phi)$ corresponds to the usual error when the downstream model is fitted on clean data (see \citet{Hastie2019Surprises}, for example). The additional term $n\sigma_0^2\phi_0/(1-\phi_0)$, proportional to the number of generations $n$, is responsible for model collapse.
% In the limit $\phi_0 \to 0$ (i.e the data-generator is not limited by finite samples per unit input dimension), the test error coincides with $\sigma^2 \phi/(1-\phi)$ which is the ordinary test error in the under-parametrized regime: model collapse disappears.

\begin{remark}
Note that the linear degeneration in test error highlighted by Equation \eqref{eq:unreg} is a direct consequence of using the same dataset size $T_0$ across the fake data generator.
Of course, if the underlying synthetic generating process has access to a larger data budget across generations,
this decay can be significantly alleviated. %\ElvisIssue{Saying "by allowing" gives the impression we control any of this. Instead say something like "if the underlying generating process is such that ..."} 
For instance, if fake data increases  gradually with the number of generations $m$ as $T_m=mT$ (and,  to simplify, $\sigma = \sigma_0$)  a trivial extension of Theorem \ref{thm:linreg} yields
$$
\begin{aligned}
    E_{test}(\widehat w_n^{pred}) & \simeq (1+\frac{1}{2} + \frac{1}{3} + \dots)E_{test}(\widehat w_0^{pred}) \\ & \simeq \log n \cdot E_{test}(\widehat w_0^{pred}),
\end{aligned}
$$
which will keep collapse at bay at the expense of largely increased training data. %At the extreme, we can allow for $n$-fold increase of fake data at the output; when $T_0=nT$,  \ElvisIssue{In fact the previous sentence "At the extreme..." is something between confusing and broken / incorrect. Maybe we should scrap it ? Maybe we want to say something like $T_n / T \to \infty$ (since this effectively kills the additional term). Anyways, I propose we scrap this ("At the extreme..."): combined with the other argument further below, I think the reader gets the point we are making.}
%we get 
%$
%E_{test}(\widehat w_n^{pred})  \simeq \frac{\sigma^2 d}{T} \simeq E_{test}(\widehat w_0^{pred})
%$. 
This does not avoid model collapse; rather, it trades additional data generation and training effort against deterioration from generations of fake data.
Thus, while for clean data increasing the dataset size n-fold leads to better scaling, with synthetic data, we not only forfeit this improvement but also experience a degeneration. Also, note that  we do not assume access to samples from any of the intermediate generation steps $\widehat w_0,\ldots,\widehat w_{n-1}$; we only train the downstream model $\widehat w_n^{pred}$ on data from the last step $\widehat w_n$. Using intermediate data could in principle, allow to effectively escape collapse, of course, again, at increased data budget.
\end{remark}

\paragraph{Low-Dimensional  Regime.} In the low-dimensional  regime (fixed $d$), Theorem \ref{thm:linreg} already predicts a change  of scaling  law   from  $E_{test}(\widehat w_0^{pred})  \asymp \sigma^2 T^{-1}$  to  $E_{test}(\widehat w_n^{pred}) \asymp \sigma^2 T^{-1} + n\sigma_0^2 T_0^{-1}$. Thus, as  the sample size $T$ is scaled up, the test error eventually plateaus at the value $n\sigma_0^2 T_0^{-1}$ and does not vanish. This phenomenon  is clearly visible  in Figure \ref{fig:linreg}. In Section \ref{sec:power}, we shall establish an analogous picture for high-dimensional regimes ($d \rightarrow \infty$).
% and both the spectrum of $\Sigma$ and coefficients of $w_0$ follow a power law.
%.\julia{Do we need this last sentence? Might be confusing.}

\paragraph{Mitigation via Regularization ?} Note that the test error of the null predictor
 $w_{null} = 0$ is $E_{test}(w_{null}) = \|w_0\|_\Sigma^2$, and so
\begin{eqnarray*}
    \frac{E_{test}(\widehat w_n^{pred})}{E_{test}(w_{null})} = \frac{1}{\text{SNR}}\frac{\phi}{1-\phi} + \frac{n}{\text{SNR}_0}\frac{\phi_0}{1-\phi_0},
\end{eqnarray*}
where $\text{SNR} := \|w_0\|_\Sigma^2/\sigma^2$ and $\text{SNR}_0 := \|w_0\|_\Sigma^2/\sigma_0^2$. We deduce that if $n \gg \mathrm{SNR}_0 / (1/\phi_0 - 1)$, then the learned model is already much worse than the null predictor! % \yunzhen{Maybe highlight it?}
This suggests that a good strategy for mitigating the negative effects on learning on AI-generated data is regularization at an appropriate level, as illustrated in Figure \ref{fig:lcrit-vs-lopt}. %\footnote{Because it shrinks the learned coefficients towards the null predictor.}. % This is indeed empirically confirmed in Figure \ref{fig:linreg}.
% \ElvisIssue{ (TODO) Theoretical analysis of the optimal level of ridge regularization. Standard Marchenko-Pastur theory should help here obtain a complete phase diagram.}

\subsection{Main Result I: A General Formula for Test Error}
We now consider the case of general ridge penalty $\lambda>0$,  and drop the requirements $T \ge d + 2$ and $T_0 \ge d + 2$.
% We still make the simplifying assumption that $T_0 \ge d + 2$, i.e each generation of the fake data-generator is based on enough examples from the previous generation.%, to allow for perfect memorization. 
% We have the following result.

Recall the definitions of $X, Y, E$ and the random matrices $R$ and $\widehat \Sigma$ appearing in \eqref{eq:ridge}. For later reference, define
\begin{align}
    Bias &:= \mathbb E\,\|\widehat \Sigma R w_0-w_0\|_\Sigma^2\label{eq:biasraw},\\
    Var &= \mathbb E\,\|RX^\top E/T\|_\Sigma^2 = \sigma^2\frac{1}{T}\trace \Sigma R^2\widehat\Sigma\label{eq:varraw}.
\end{align}
% \julia{$E$ is undefined here! What is $E$?}
These are respectively the bias and variance terms in the classical bias-variance decomposition of the test error
\begin{eqnarray}
\label{eq:raw}
E_{test}^{clean} := Bias + Var,
\end{eqnarray}
for standard ridge regression fitted on clean data from the true data distribution $P_{\Sigma,w_0,\sigma^2}$ (e.g., see \citet{Hastie2019Surprises}). % \julia{Do we want to add  reference(s)?}
\begin{theorem}
For an $n$-fold fake data generation process, the test error of a ridge predictor $\widehat w_n^{pred}$ based on a sample of size $T \ge 1$ with regularization parameter $\lambda$ is given by
    \begin{eqnarray}
\left.
\begin{aligned}
    E_{test}(\widehat w_n^{pred}) &= \textcolor{red}{\widetilde{Bias}} + Var + \textcolor{red}{n\sigma_0^2\rho}, \\
    \widetilde{Bias} &=  \mathbb E\,\|\widehat\Sigma R Q_{n-1} w_0 - w_0\|_\Sigma^2,\\
    % \Delta_1 &:= \mathbb E\,[\|\Delta w_0\|_\Sigma^2],\\
    % \Delta_2 &:= 2\mathbb E\, [\Delta w_0^\top \Sigma (I-S R)w_0],\\
    \rho &:= \frac{1}{n}\sum_{m=0}^{n-1}\mathbb E\,\trace C_{n-1,m} \widehat\Sigma R\Sigma R\widehat\Sigma,
\end{aligned}
\right\}
\label{eq:b2}
\end{eqnarray}
where $Var$ is as given in \eqref{eq:varraw} and  $C_{k,m}:=\overline Q_{k,m}\overline Q_{k,m}^\top$ for $\overline Q_{k,m} = Q_{k,m}X_m^\dagger$, $Q_{k,m} := P_kP_{k-1}\ldots P_m$, $Q_k := Q_{k,0}=P_k P_{k-1}\ldots P_0$, with $P_m=X_m^\dagger X_m$ being the orthogonal projection matrix onto the subspace of $\mathbb R^d$ spanned by the rows of $X_m$. %; $\Delta w_0 := P_0 w_0 -w_0\in \mathbb R^d$.

In particular, if $T_0 \ge d+2$ (under-parametrized data-generator),  then $\widetilde{Bias} = Bias$ as given in \eqref{eq:biasraw}, and 
\begin{eqnarray}
\left.
\begin{aligned}
E_{test}(\widehat w_n^{pred}) &\simeq E_{test}^{clean} + \textcolor{red}{n\sigma_0^2 \rho},\\
% E_{test}^{clean} &= Bias + Var,\\
    \rho &= \frac{1}{T_0-d-1}\mathbb E \trace \Sigma^{-1}\widehat \Sigma R\Sigma \widehat\Sigma R.
\end{aligned}
\right\}
\end{eqnarray}
\label{thm:generalcov}
\end{theorem}
In the second part of the theorem, the term $E_{test}^{clean}$ (introduced earlier in \eqref{eq:raw}) corresponds to the usual test error when the downstream model is trained on real (not fake) data, for which well-known formulae exist in a variety of scenarios (see Proposition \ref{prop:classicalrige}).
% Even in this special case, what is new is the term $n\sigma_0^2\rho$, where $n$ is the number of generations.

\begin{remark}
We shall later show in Theorem \ref{thm:rho} that $\widetilde{Bias} \ge Bias + \Delta Bias$, where $\Delta Bias \ge 0$ in the appropriate asymptotic limit, with equality if $T_0 \ge d+2$ (the {\em under-parametrized} regime). Thus, apart from the variance term, an over-parametrized ($T_0 < d +2$) synthetic data-generator harms the bias term of the test error of downstream models. In contrast, an under-parametrized synthetic data-generator ($T_0 \ge d+2$) only harms the variance. The increase in bias suffered in the over-parametrized regime will be precisely quantified in Section \ref{sec:cata}, and shown to be an increasing function of the number of generations $n$.
\end{remark}

The test error decomposition in Theorem \ref{thm:generalcov} is thus of the promised form \eqref{eq:cata}, with $Scaling = \sigma_0^2\rho$. This additional term means that there is competition between the usual test error $E_{test}^{clean}$ and the additional term induced by the fake labeling process.
Understanding the interaction of these two terms is key to demystifying the origins of model collapse.

%Importantly, the term $\rho$ only depends on the random design matrices $X_0$ and $X$, and the polylation covariance matrix $\Sigma$.  In an appropriate asymptotic regime, it will only depend on: $T$, $d$, and the spectral properties of $\Sigma$. Observe that we always have $\rho \le d/(T_0-d)$, a consequence of the fact that all the eigenvalues of the positive-definite matrix $(S + \lambda I_d)^{-1}S$ are at most $1$. Moreover,  $\rho \to 0$ in the limit $\lambda \to \infty$ (corresponding to the null predictor).

\paragraph{Low-Dimensional Limit.}
Observe that if $d$ is fixed and $T \to \infty$, then the empirical covariance matrix $\widehat\Sigma$ converges to\footnote{e.g. weakly, w.r.t. operator-norm.} its population version $\Sigma$, and so for $T_0 \ge d + 2$, we have
$$
    \rho \simeq \frac{\trace \Sigma^2(\Sigma + \lambda I_d)^{-2}}{T_0-d} = \frac{\mathrm{df}_2(\lambda)}{T_0-d},
$$
where for any $\lambda \ge 0$ and $m \in \mathbb N_\star$, $\df_m(\lambda)$ is the $m$th order "degree of freedom" of the covariance matrix $\Sigma$ is given by
$$
\mathrm{df}_m(\lambda) = \mathrm{df}_m(\lambda; \Sigma) := \trace\Sigma^m(\Sigma + \lambda I_d)^{-m}.
$$
Note that $\mathrm{df}_m(\lambda) \le d$ always. In the high-dimensional setting (where $d$ can grow beyond $T_0$), the precise analysis of $\rho$
% the trace term $\rho$ in Theorem \ref{thm:generalcov}
will be carried out via random matrix theory (RMT).
% We shall need to carefully split the analysis into different sub-cases depending on the high-dimensional behavior of the covariance matrix $\Sigma$.

\subsection{High-Dimensional Regimes in the RMT limit}
In order to analyze the trace term $\rho$ appearing in \eqref{eq:b2}, we need some tools from RMT, and ultimately obtain analytic formulae for $E_{test}(\widehat w_n^{pred})$ in Theorem \ref{thm:rho}. Such tools have been used extensively to analyze anisotropic ridge regression \citet{Richards2021, Hastie2019Surprises,bach2023highdimensional}.
A standard reference on RMT is \citet{BaiRMTBook}. %\ElvisIssue{Add more refs!}

% Define a function $\varphi:\mathbb C\setminus \mathbb R_+ \to \mathbb R_+$ implicitly as the unique nonnegative solution to the following equation
% \begin{eqnarray}
%     \frac{1}{\varphi(z)} + z = \phi\int_0^\infty \frac{t\mathrm{d}\mu(t)}{1+t\varphi(z)}.
% \end{eqnarray}

\textbf{Random Matrix Equivalents.}
For any sample size $T \ge 1$, define an increasing function $\lambda \to \kappa(\lambda,T)$ implicitly by
\begin{eqnarray}
\begin{aligned}
    \kappa(\lambda,T) - \lambda & = \kappa(\lambda,T)\cdot \mathrm{df}_1(\kappa(\lambda,T))/T\,.
    \label{eq:kappa}
\end{aligned}
\end{eqnarray}
% % It should then be clear that
% \begin{eqnarray}
%     \kappa(\lambda,T)\varphi(-\lambda) \simeq 1,\text{ for all }\lambda \ge 0.
% \end{eqnarray}
The effect of ridge regularization at level $\lambda \ge 0$ is to improve the condition of the empirical covariance matrix $\widehat \Sigma$; what the $\kappa$-function does is translate this into regularization on $\Sigma$ at level $\kappa(\lambda,T)$, so as control the capacity of the former, i.e. the "effective dimension" of the underlying problem. Quantitatively, there is an equivalence of the form
$$
\mathrm{df}_1(\lambda; \widehat \Sigma) \approx \mathrm{df}_1(\kappa(\lambda,T); \Sigma).
$$
Roughly speaking, RMT is the business of formalizing such a relationship and derivatives (w.r.t.~$\lambda$) thereof.

\textbf{Example: Isotropic Covariance.}
As an illustration, note that $\mathrm{df}_m(\lambda) \equiv d/(1+\lambda)^m$ (polynomial decay) in the isotropic case where $\Sigma=I_d$. Consequently, we have 
$$
\kappa(\lambda,T) - \lambda = \phi \cdot \kappa(\lambda,T)/(1+\kappa(\lambda,T)),\text{ with }\phi := d/T.% \phi\kappa(\lambda,T) / (1 + \kappa(\lambda,T)).
$$
In this case, it is easy to obtain the following well-known formula for $\kappa=\kappa(\lambda,T)$:
% $\kappa(\lambda,T) = 1/\varphi(-\lambda)$ and 
\begin{eqnarray}
    \kappa = \frac{1}{2}\left(\lambda + \overline\phi + \sqrt{(\lambda+\overline\phi)^2 + 4\lambda}\right),\text{ with }\overline\phi := \phi-1,
    \label{eq:MP}
\end{eqnarray}
% where $\overline \phi := 1-\phi$,
which is reminiscent of the celebrated Marchenko-Pastur law \citep{MarcenkoPastur}.

% \textbf{Assumptions.}
% We will temporarily work under the following standard assumption:
% % \ElvisIssue{Cite standard refs, including papers by Dobriban, Couillet, and Pennington (potential reviewers) !}
% \begin{assumption}
% $\Sigma=\Sigma_d$ (i.e a sequence of covariance matrix indexed by the dimensionality $d$) has spectrum bounded away from $0$ uniformly w.r.t $d$. Moreover, the empirical spectral distribution $\sum_{j=1}^d \delta_{\lambda_j}$ of $\Sigma$ converges in the limit $d \to \infty$, to a  compactly-supported distribution $\mu$ on $\mathbb R^+$.
% \end{assumption}
% \label{ass:boundedspec}
Furthermore, we shall work in the following so-called proportionate asymptotic scaling which is a standard analysis based on random matrix theory (RMT):
\begin{eqnarray}
\label{eq:proportionate}
    T,d \to \infty,\, d/T \to \phi,\, \,\|\Sigma\|_{op},\|\Sigma^{-1}\|_{op} = O(1).
\end{eqnarray}
Later in Section \ref{sec:power} when we consider power-law spectra, this scaling will be extended to account for the more realistic case where $d$ and $T$ are of the same order on log scale, i.e
% larger than one another, i.e 
% \begin{mdframed}
%\textbf{Polynomial Scaling Regime.}
    \begin{eqnarray}
T,d \to \infty,\, d^{1/C} \lesssim T \lesssim d^C,\, \|\Sigma\|_{op},\|\Sigma^{-1}\|_{op} = O(1),
\label{eq:nonproportionate}
\end{eqnarray}
for some absolute constant $C \ge 1$.
% \end{mdframed}
Such non-proportionate settings are covered by the theory developed in \citet{Knowles2017,WeiMoreThanToy}. For clarity of presentation, even in this more general regime \eqref{eq:nonproportionate}, we will still continue to write $\phi_0 := d/T_0$ and $\phi := d/T$.

\textbf{Bias-Variance Decomposition.}
With everything now in place, let us recall for later use, the following classical bias-variance decomposition for ridge regression (for example, see \citet{Richards2021,Hastie2019Surprises,bach2023highdimensional}):
\begin{proposition}
% Under Assumption \ref{ass:boundedspec}, in the limit  \eqref{eq:proportionate}
In the RMT limit \eqref{eq:nonproportionate}, the test error of a ridge predictor $w(\lambda)$ based on $T$ iid samples from the true data distribution $P_{\Sigma,w_0,\sigma^2}$ is given by
    \begin{align}
        E_{test}(w(\lambda)) &= \mathbb E\,\|w(\lambda)-w_0\|_\Sigma^2 \simeq Bias + Var, \nonumber\\
        Bias &\simeq \frac{\kappa^2 w_0^\top \Sigma(\Sigma + \kappa I)^{-2} w_0}{{1-\mathrm{df}_2(\kappa)/T}}\label{eq:bias},\\
        Var &\simeq  \frac{\sigma^2\mathrm{df}_2(\kappa)}{T}\cdot \frac{1}{1-\mathrm{df}_2(\kappa)/T},
        \label{eq:var}
    \end{align}
    where $\kappa = \kappa(\lambda,T)$ is as given in \eqref{eq:kappa}.

In particular, in the isotropic case where $\Sigma = I_d$,  we have
$$
    E_{test}(w(\lambda)) \simeq \frac{\kappa^2 \|w_0\|_2^2 + \sigma^2 \phi}{(1+\kappa)^2 - \phi},
$$
    where $\phi:=d/T$ and $\kappa = \kappa(\lambda,T)$ is as given in \eqref{eq:MP}.
    \label{prop:classicalrige}
\end{proposition}
% For completeness, the second part of the result is proved in the appendix.

\subsection{Main Result II: Analytic Formula for Test Error}
% Putting Theorem \ref{thm:generalcov}, Proposition \ref{prop:b2}, and Proposition \ref{prop:classicalrige} together, we obtain
The following result gives the test error for the downstream ridge predictor $\widehat w_n^{pred}$ defined in \eqref{eq:ridge}, in the context of fake training data, and will be heavily exploited later to obtain precise estimates in different regimes. For later reference, define $\Delta Bias \ge 0$ by
\begin{eqnarray}
    \Delta Bias := \mathbb E\,\|\widehat \Sigma R(Q_{n-1}w_0-w_0)\|_\Sigma^2,
    \label{eq:deltabias}
\end{eqnarray}
where the matrix $Q_{n-1}$ is as defined in Theorem \ref{thm:generalcov}.
 \begin{theorem}
% Suppose Assumption \ref{ass:boundedspec} is in order.
For an $n$-fold fake data-generation process, the test error of a ridge predictor $\widehat w_n^{pred}$ based on a sample of size $T$ with regularization parameter $\lambda$, is given in the RMT limit \eqref{eq:nonproportionate} by % in the polynomial scaling regime $T,T_0,d \gg 1$ s.t. $\log T,\log T_0 \asymp \log d$,  is given by
    \begin{eqnarray}
    E_{test}(\widehat w_n^{pred})  \simeq \textcolor{red}{\widetilde {Bias}} + Var + \textcolor{red}{n\sigma_0^2 \rho},
    \end{eqnarray}
    where $Var$ and $\rho$ are as given in Theorem \ref{thm:generalcov}, and $\widetilde{Bias}$ satisfies
\begin{eqnarray*}
    \begin{aligned}
\widetilde{Bias } &\ge Bias + \textcolor{red}{\Delta Bias} \ge Bias\\
&\text{(with equality if }T_0 \ge d\text{)},
    \end{aligned}
    \end{eqnarray*}
and $\Delta Bias$ as given in \ref{eq:deltabias}. 

Furthermore, if one of the  following conditions holds
\begin{eqnarray}
\label{eq:simplifier}
T_0 \ge d\text{ OR }X_n = X_0\text{ for all }n,
\end{eqnarray}
then, we have the following explicit formula for $\rho$
\begin{eqnarray*}
\begin{aligned}
        \rho &= \dfrac{\trace \Sigma^4(\Sigma + \kappa_0 I)^{-2}(\Sigma + \kappa I)^{-2}}{T_0 - \df_2(\kappa_0)}\\
        &\quad + \dfrac{\kappa^2 \trace\Sigma^2 (\Sigma + \kappa_0 I)^{-2}(\Sigma + \kappa I)^{-2}}{T_0-\df_2(\kappa_0)}\cdot \dfrac{\df_2(\kappa)}{T-\df_2(\kappa)},
\end{aligned}
\end{eqnarray*}
    with $\kappa=\kappa(\lambda,T)$ and $\kappa_0 := \kappa(0,T_0)$ are as given in \eqref{eq:kappa}.
\label{thm:rho}
\end{theorem}
Instructively, the term $\Delta Bias$ measures how biased the synthetic data-generation process away from ground-truth model $w_0$. This term disappears if the generator was fitted on sufficiently many samples (i.e. if $T_0 \ge d$). More quantitatively, when $T_0 < d$ and $X_n=X_0$, it is easy to see that $\Delta Bias \ge \mathbb E\, [\|\Sigma^{1/2}\widehat\Sigma R\|_{op}^2] \cdot Bias_0$, where $Bias_0 := \mathbb E\,\|P_0w_0-w_0\|_2^2$ measures the inability due to lack of enough data, of the first generation ($n=1$) to reliably estimate $w_0$ even in the absence of noise ($\sigma_0=0$) in the data-generating process.
% \julia{Check if this is still true now!}
This gap propagates over to higher generations of the process. The situation is illustrated in Figure \ref{fig:harmsway}. In the case where $T_0 < d$ and the $X_n$'s are independent, we shall see in Section \ref{sec:cata} that this increase in bias actually grows with $n$, even when $\sigma_0=0$.

\subsection{Under-Parametrized Fake Data-Generator}
Observe that if $T_0 \ge d$, then $P_0=I_d$ a.s., leading to $\widetilde{Bias} = Bias$ (given as in formula \eqref{eq:bias}), and 
$\kappa_0 = 0$ in Theorem \ref{thm:generalcov},  and 
\begin{eqnarray}
    \rho = \frac{\mathrm{df}_2(\kappa)}{T_0-d} + \frac{\kappa^2 \trace\,(\Sigma + \kappa I)^{-2}}{T_0-d}\dfrac{\df_2(\kappa)}{T-\df_2(\kappa)},
    \label{eq:underparam-rho}
\end{eqnarray}
with $\kappa$  as in Theorem \ref{thm:generalcov}.
We have the following corollary.
\begin{corollary}
\label{cor:previous-main-thm}
Consider the setting of Theorem \ref{thm:generalcov}. If $T_0 \ge d$ additionally, then it holds in the RMT limit \eqref{eq:nonproportionate} that
\begin{eqnarray*}
    E_{test}(\widehat w_n^{pred}) \simeq Bias + Var + n\sigma_0^2 \rho,
\end{eqnarray*}
 where $Bias$ and $Var$ are as given in formulae \eqref{eq:bias} and \eqref{eq:var} respectively, and $\rho$ is as given in \eqref{eq:underparam-rho}.

In the special case of isotropic features, it holds that
$$
\begin{aligned}
&Bias + Var \simeq \frac{\kappa^2 \|w_0\|_2^2 + \sigma^2 \phi}{(1+\kappa)^2 - \phi},\\
    &\rho \simeq \frac{\phi_0}{1-\phi_0}\left(\frac{1}{(1+\kappa)^2} +  \frac{1}{(1+\kappa)^2}\frac{\phi\kappa^2}{(1+\kappa)^2 - \phi}\right),
    \end{aligned}
$$
with $\phi:=d/T$, $\phi_0:=d/T_0$, and $\kappa = \kappa(\lambda,T)$ as in \eqref{eq:MP}.
\end{corollary}
Such a result
% , empirically illustrated in Figures \ref{fig:linreg} and \ref{fig:krreg} \julia{Do you really mean Fig 2? Because by its title Fig 2 refers to Sec 5, power-law spectra.},
gives us the needed analytical handle for understanding $n$-fold model collapse in terms of all problem hyper-parameters (covariance spectrum, regularization, label-noise level, etc.). 
% \yunzhen{Add sentences to call back $\rho$. How is the proposition related to understanding $\rho$. Then end with "with such a result, we could analytically characterize the n generation model collapse"}
% The result is empirically illustrated in Figure \ref{fig:linreg}. 

% \yunzhen{Add some interpretations on the figure}
% \ElvisIssue{ (TODO) Use the analytic expressions provided in theorem to deduce a qualitative / quantitative picture of the optimal ridge parameter, at least in the isotropic case.}

% \ElvisIssue{(TODO) Study the diatomic setup where $\Sigma$ has $\pi d$ eigenvalues equal to $a$ and $(1-\pi)d$ eigenvalues equal to $b \ne a$, i.e $\mu = \pi \delta_a + (1-\pi)\delta_b$. Some interesting phenomena may prevail.
% }

% \ElvisIssue{(TODO) Study the case where $\Sigma_0 \ne \Sigma$.}

\subsection{Model Collapse in  the Absence of Label Noise}
\label{sec:cata}

We now consider the over-parametrized regime, where the different iterations of the synthetic data-generator (refer to the illustration in Figure \ref{fig:teaser}) are fitted on insufficient data. For simplicity of exposition, we restrict our presentation to isotropic covariance $\Sigma=I_d$. Since we will be focusing on the possible increase $\Delta Bias$ in bias (defined in \eqref{eq:bias}) due to $n \ge 1$ generations as predicted by Theorem \ref{thm:rho}, we further restrict ourselves to the noiseless regime where the fake data-generating process has no label noise, i.e $\sigma_0=0$. %  ; considering general covariance matrices produces to same qualitative picture.
Thanks to Lemma \ref{lm:whatrep}, we know that the generation-$n$ fake labelling vector $\widehat w_n$ is in this setting given explicitly by the formula
\begin{eqnarray}    
\widehat w_n=P_{n-1}P_{n-2}\ldots P_0 w_0,
\label{eq:dynamo}
\end{eqnarray}
where, as usual, $P_m \in \mathbb R^{d \times d}$ is the orthogonal projection matrix onto the subspace of $\mathbb R^d$ spanned by the rows of $X_m$. Further, for simplicity we will assume $T=T_n>d$, i.e the downstream model has access to enough data.

We shall focus on two important special cases.

\paragraph{Dependent Case.} We first consider the case where $T_m=T_0<d$ and $X_m=X_0$ for all $m \le n-1$. It is clear that \eqref{eq:dynamo} reduces to $\widehat w_n = P_0 w_0$. We have the following result.
\begin{theorem}
In the limit $\lambda \to 0^+$ and $d,T_0 \to \infty$ with $d/T_0 \to \phi_0 > 1$, it holds that
\begin{align}
\|\widehat w_n\|^2 &\simeq \|w_0\|^2/\phi_0,\\
\Delta Bias &\simeq \|w_0\|^2(1-1/\phi_0).
\label{eq:deltaBiasEst}
\end{align}
\label{thm:P0}
\end{theorem}
We see that in this setting, the increase in bias $\Delta Bias \simeq (1-1/\phi_0)\|w_0\|^2$ brought about by synthetic data is a positive constant which does not grow with the number of generations $n \ge 1$. This increase in bias (i.e compared to training on clean data) is due to the fact that, with probability $1$, the random subspace of $\mathbb R^d$ spanned by $X_0$ does not contain the ground truth model $w_0$. The expression is nothing but a RMT estimate of $\|P_0w_0-w_0\|^2$, i.e the squared norm of the projection of $w_0$ onto the orthogonal complement of this subspace. The result is illustrated in Figure \ref{fig:harmsway}(a).

\paragraph{Independent Case.}
For our second example, we remove the assumption that $T_m=T_0$ and $X_m=X_0$ for all $m \le n-1$ considered in the previous case (Theorem \ref{thm:P0}). We instead assume that
\begin{itemize}
\item  the $X_m$'s are assumed to be independent, and
\item we are in the following high-dimensional limit
\begin{eqnarray}
\label{eq:longtrain}
d,T_1,\ldots,T_{n-1} \to \infty,\quad d/T_m \to \phi_m,
\end{eqnarray}
where $\phi_1,\ldots,\phi_{n-1}$ are positive constants.
\end{itemize}
% \julia{Problem: We defined $\widehat w_0=w_0$. So, you need to say $n \ge 1$ here, right?}
We have the following theorem. 
\begin{theorem}
\label{thm:cata}
In the limit \eqref{eq:longtrain} and $\lambda \to 0^+$, it holds that
\begin{align}
\|\widehat w_n\|^2 &\simeq \|w_0\|^2\prod_{m=0}^{n-1}\min(1/\phi_m, 1),\\
\Delta Bias &\simeq \|w_0\|^2\left(1-\prod_{m=0}^{n-1} \min(1/\phi_m,1)\right).
\end{align}
In particular, if $n \to \infty$ such that infinitely many of the $\phi_m$'s are  $>1$, then $\widehat w_n \to 0$ and $\Delta Bias \to \|w_0\|^2$.
\end{theorem}
The theorem predicts that a sequence of over-parametrized fake data-generators $(\widehat w_n)_n$ collapses to zero (and thus, effectively escapes from the ground truth model $w_0$). Consequently, the downstream model $\widehat w_n^{pred}$ convergences to a Gaussian process around zero, instead of the true model $w_0$, leading to an increase in the bias term of the test error!
% This result is in the spirit of \cite{bertrand2023stability}.

For example if $\phi_m = \phi_0 > 1$ for all $m \le n-1$, then
$$
\Delta Bias \simeq (1-\phi_0^{-n})\|w_0\|^2,
$$
which grows exponentially fast towards $\|w_0\|^2$, the test error of the null predictor. This compounding effect is due to the fact that in \eqref{eq:dynamo}, each projection $P_m$ spins the fake data labelling vector $\widehat w_n$ away from the ground-truth $w_0$ even further. The result is illustrated in Figure \ref{fig:harmsway}(b).

\section{The Case of Heavy Tails (Power Law)}
\label{sec:power}
% \ElvisIssue{Motivate why this regime is relevant for studying real ML setups: DL / LLMs / etc.}

%\bbb{%Neural scaling laws posit that test loss decreases linearly with increases in data and model size on a logarithmic scale. These laws are crucial for the success of large-scale language models. 
%Now, consider a variant of the distribution \eqref{eq:regprob} where the spectral decomposition follows a power-law rule for $d \to \infty$. This setting is considered in \citet{Caponnetto2007OptimalRF,Richards2021,SimonNTK2021,Cui_2022}. %and is recently used to provide theoretical understandings for scaling laws in kernel methods \citet{Cui_2023}.}
Now, consider a variant of the distribution \eqref{eq:regprob}, in the setting considered in \citet{Caponnetto2007OptimalRF,Richards2021,SimonNTK2021,Cui_2022}, for $d \to \infty$. 
Let the spectral decomposition of $\Sigma$ be
$$
\Sigma = \lambda_1 v_1 v_1^\top + \ldots + \lambda_d v_d v_d^\top,
$$
 where $\lambda_1 \ge \ldots \ge  \lambda_d \ge 0$ are the eigenvalues and $v_1, \ldots, v_d \in \mathbb R^d$ are the eigenvectors.
 % Compactness means $\sum_i \lambda_i < \infty$.
 For any $j$, define a coefficient $c_j := w_0^\top v_j$, i.e the projection of $w_0$ along the $j$th eigenvector of $\Sigma$.

\begin{figure*}[!htb]
    \centering
    \begin{subfigure}{\textwidth}
        \centering
        \includegraphics[width=0.8\linewidth]{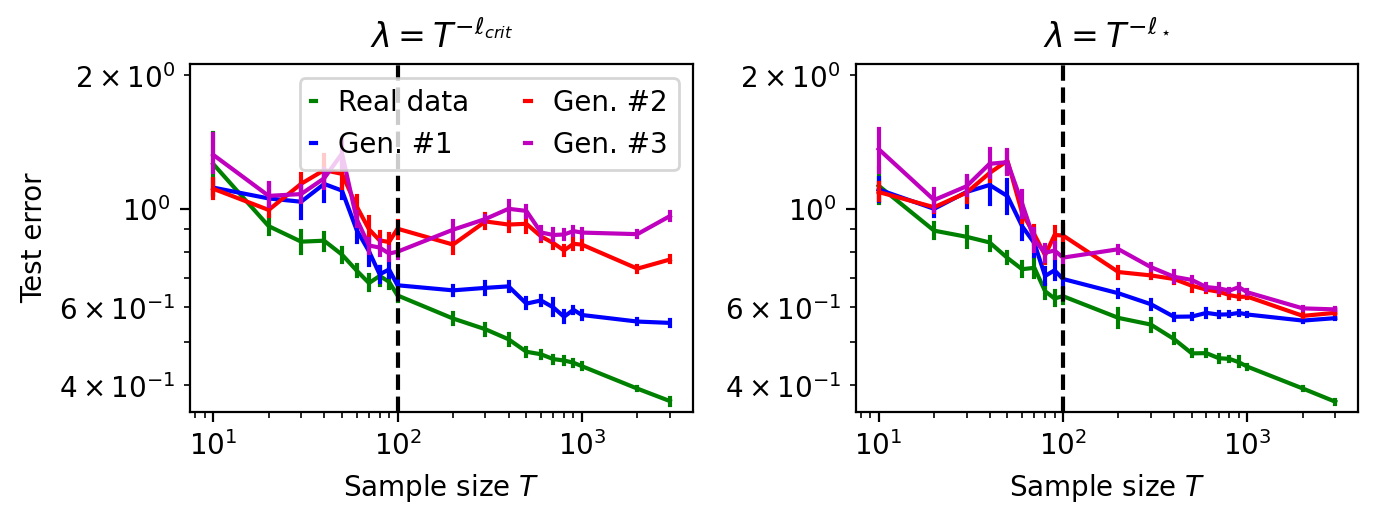}
        \caption{RBF kernel (bandwidth = $10^{-4}$)}
        % \label{fig:mnist_krr_rbf}
    \end{subfigure}
    % Add some vertical space between the figures if necessary
    % \vspace{1cm}
    \begin{subfigure}{\textwidth}
        \centering
        \includegraphics[width=0.8\linewidth]{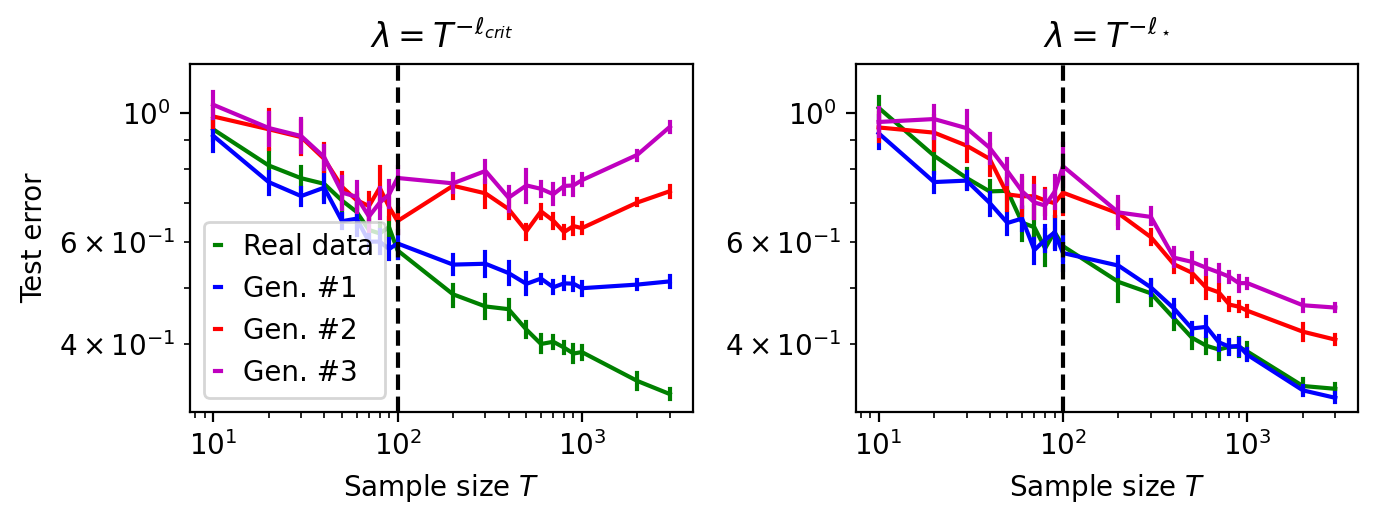}
        \caption{Polynomial kernel (degree = $5$, bandwidth = $10^{-3}$)}
        % \label{fig:mnist_krr_polynomial}
    \end{subfigure}
    \vspace{-.6cm}
    \caption{
    \textbf{Demystifying model collapse in kernel ridge regression (power-law covariance spectrum) on MNIST}. Here, we use adaptive regularization $T^{-\ell}$ for different values of the exponent $\ell \ge 0$ (see Section \ref{sec:exp} for full experimental setup). \textbf{Top row: } RBF kernel. \textbf{Bottom row: }polynomial kernel.
     In each plot, we show test error curves as a function of sample size $T$, from different generations ($n$) of fake data. The broken vertical line corresponds to $T=T_0$, where $T_0$ is the number of samples (from the true data distribution) which was used to train the label faker. The value of the exponent regularization $\ell=\ell_{\star}$ (broken curves) is the optimal value in the presence of iterative data relabeling, while $\ell=\ell_{crit}$ (solid curves) corresponds to the optimal value  without iterative re-labelling (i.e $n=0$) proposed in \citet{Cui_2022} (see \eqref{eq:lcrit}). Specifically, we take  $\ell_\star=(b-a)\ell_{cirt} = b\ell_{crit}$, where $b=\log T_0 / \log T$ (so that $T_0 = T^b$), as proposed in  Theorem \ref{thm:darkseid}, formula \eqref{eq:lopt}.
     Notice how the effect of fake data makes the test error become non decreasing in sample size $T$. This is effectively a collapse of the learned model.
     % The broken diagonal line is predicted by \citet{Cui_2022}, and would correspond to $T_0=\infty$.
    }
     \vspace{-.4cm}
    \label{fig:lcrit-vs-lopt}
\end{figure*}

We shall work under the following spectral conditions
\begin{eqnarray}
\left.
\begin{aligned}
    &\textbf{(Capacity Condition) } \lambda_j \asymp j^{-\beta}\text{ for all }j \in [d],\\
    &\textbf{(Source Condition) }\|\Sigma^{1/2-r} w_0\| = O(1),% \text{ i.e }\sum_i \lambda_i^{-2r}c_i^2 < \infty.
    \end{aligned}
        \right\}
    \label{eq:powerlaw}
\end{eqnarray}
where $\beta > 1$ and $r \ge 0$. The parameter $r$ measures the amount of dispersion of $w_0$ relative to the spectrum of $\Sigma$; a large value of $r$ means $w_0$ is concentrated only along a few important eigen-directions (i.e.~ the learning problem is easy). For later convenience, define $\underline r$ and $\delta$ by
$$
\underline r := \min(r,1),\quad \delta := 1+\beta(2r-1).
$$
As noted in \citet{Cui_2022}, the source condition in \eqref{eq:powerlaw} is satisfied if $c_j \asymp j^{-\delta/2}$ for all $j \in [d]$.
% For example, if $r = r_0 := (1-1/\beta)/2$, then $\delta = 0$. In general, note that $\delta$ can be negative; this is precisely when $r < r_0$. Note that $\beta-1 + \delta = \beta + \beta(2r-1) = 2\beta r$.

Consider adaptive ridge regularization strength of the form
\begin{eqnarray}
\label{eq:reguexpo}
\lambda = \lambda(T) \asymp T^{-\ell},
\end{eqnarray}
for fixed $\ell \ge 0$. The case where $\ell=0$ corresponds to non-adaptive regularization; otherwise, the level of regularization decays polynomially with the sample size $T$. Define
\begin{eqnarray}
\label{eq:lcrit}
\ell_{crit} :=  \beta / (1+2\beta \underline r). 
\end{eqnarray}
In \citet{Cui_2022} KRR under normal circumstances (corresponding to $n=0$, i.e.~ no fake data) was considered and it was shown that this value for the regularization exponent in \eqref{eq:reguexpo} is minimax-optimal for normal test error in the noisy regime ($\sigma>0$), namely $E_{test}(\widehat w_0^{pred}) \asymp T^{-c}$,
where
 $$
 c:=\frac{2\beta\underline r}{2\beta \underline r + 1} \in (0,1).
 $$
 This represents a crossover from the noiseless regime where it was shown that the test error scales like $E_{test}(\widehat w_0^{pred}) \asymp T^{-2 \beta \underline r}$, a much faster rate. We shall show that the picture drastically changes in the context of training on fake data considered in this manuscript for the purpose of understanding model collapse \citep{shumailov2023curse}.

\begin{remark}
    Unlike \citet{Cui_2022} which considered the proportionate scaling limit \eqref{eq:proportionate} for input dimension $d$ and sample size $T$, we shall consider the more general (and more realistic) polynomial scaling limit \eqref{eq:nonproportionate}, and invoke the tools of so-called \emph{anisotropic local RMT} developed in \citet{Knowles2017} to compute deterministic equivalents for quantities involving the spectra of random matrices.
\end{remark}
\subsection{Main Result III: A "Collapsed" Scaling Law}
The following result shows that model collapse is a modification of usual scaling laws induced by fake data. All proofs of this section can be found in Appendix \ref{app:B}. Here, we restrict to the case $T_0 \ge d+2$ to make the results easier to present. This condition can be removed as in Theorem \ref{thm:rho}.
\begin{theorem}
Consider $n$-fold fake-data generation with sample size $T_0 \ge d + 2$. % and set $\phi_0 := d/T_0 \in (0,1)$. 
For a ridge predictor $\widehat w_n^{pred}$ given in \eqref{eq:ridge} based on a fake data sample of size $T$, with regularization parameter $\lambda=\lambda(T)$ tuned adaptively as in \eqref{eq:reguexpo} with exponent $\ell \in [0,\beta)$, the test error satisfies the following scaling law in the RMT limit \eqref{eq:nonproportionate}:
\begin{equation}
\begin{aligned}
    E_{test}(\widehat w_n^{pred}) & \asymp \max(\sigma^2,T^{1-2\underline r\ell - \ell/\beta)})\cdot T^{-(1-\ell/\beta)} \nonumber \\ & + \textcolor{red}{\frac{n\sigma_0^2}{1-\phi_0}\max\left(T/T_0,\phi_0\right) \cdot T^{-(1-\ell/\beta)}}.
    \end{aligned}
    \label{eq:tron}
\end{equation}
% where  $\phi_0 := d/T_0$.
\label{thm:darkseid}
\end{theorem}
% \ElvisIssue{In practice (e.g on MNIST) we don't know $\phi_0$, since we don't know $d$. Maybe $d$ is "effective dimension of data manifold" ?}

\subsection{Optimal Regularization for Mitigating Collapse}
Let us provide an instructive interpretation of the result.

% The regime $0 \le \ell \le \ell_{crit} $, the second term in the error rate is of order $T^{-2\ell \underline r}$.

\textbf{Noiseless Regime.}
Suppose $\sigma = 0$ (or equivalently, exponentially small in $T$) and $\phi_0 \in (0,1)$ is fixed, and consider a number of generations $n$ such that $n\sigma_0^2 \asymp T^a$, where $0 \le a \le 1-\ell/\beta \le 1$. Note that $a=0$ corresponds to a constant number of generations. 
Also take $T_0 = T^b$, for some constant $b \in (0,\infty)$. According to Theorem \ref{thm:darkseid}, if we want to balance out the model-collapsing negative effect of training on fake data, we should chose $\ell$ so as to balance the second term $n (T/T_0) T^{-(1-\ell/\beta)} = T^{-(b-\ell/\beta-a)}$ and the first term $T^{-2\ell \underline r}$. This gives the following result.
% , i.e $b-\ell/\beta - a = 2\ell\underline r$, giving
\begin{corollary}
In the setting of Theorem \ref{thm:darkseid} with $T_0 \asymp T^b$ and $n \asymp T^b$, the optimal exponent of the ridge regularization parameter in \eqref{eq:reguexpo} is $\ell=\ell_\star$, where
 \begin{eqnarray}
     \ell_\star  = \min((b-a)\ell_{crit},\beta),
     \label{eq:lopt}
 \end{eqnarray}
 and $\ell_{crit}$ is as in \eqref{eq:lcrit}, with corresponding optimal test error
 $$
 \inf_{\ell \ge 0}E_{test}(\widehat w_n^{pred}) \asymp E_{test}(\widehat w_n^{pred})\big|_{\ell = \ell_\star} \asymp T^{-(b-a)c}.
 $$
 \label{cor:lopt}
 \end{corollary}
 % Note that $0 \le \ell_\star < \ell_{crit}$ provided $a>b-1$; in this case, $\ell_\star$ is $1/(b-a) \in (1,\infty)$ times smaller than the regularization exponent prescribed in \citet{Cui_2022}, namely $\ell = \ell_{crit}$. If $a > b-1$, the previous statement means we need much larger regularization than would normally be required optimally.
 % The fast rate $T^{-2\beta\underline r}$ (without AI pollution) degrades to $T^{-(b-a)c}$.
 % In comparison, using $\ell = \ell_{crit}$ as prescribed in \citet{Cui_2022}, the error rate instead degrades to $T^{-(c+1-b-a)}$.
 % Note that $(b-a) c > c+1-b-a$ (\ElvisIssue{For which values of $(a,b)$?}, and so our new regularization exponent $\ell_\star$ always beats \citet{Cui_2022} under AI pollution, by orders of magnitude!
 \vspace{-.75cm}
Observe that when $(b-a)c < 2\beta \underline r$, which is the case when $n=O(1)$, $r \ge 1$ and $b\le a+1$, this corresponds to the condition $T \gtrsim T_0$. The above result represents a crossover from the fast rate $E_{test}(\widehat w_0^{pred}) \asymp T^{-2\beta \underline r}$ in the case of training on clean data  \citet{Cui_2022}, to a much slower rate $E_{test}(\widehat w_n^{pred}) \asymp T^{-(b-a)c}$, attained by the adaptive regularization $\lambda  \asymp T^{-\ell_\star}$, which is  % in fact
optimal in this setting.
Furthermore, in this setting if we still use $\lambda \asymp T^{-\ell_{crit}}$ as proposed in \citet{Cui_2022} in the clean data setting, Corollary \ref{cor:lopt} predicts that
$$
E_{test}(\widehat w_n^{pred}) \gtrsim T^{-(b-\ell_{crit}/\beta-a)} = T^{-(c+b-a-1)},
$$
which diverges to infinity if $b \ge a+1-c$. This is a catastrophic form of model collapse, and is empirically illustrated in Figures \ref{fig:krreg} and \ref{fig:lcrit-vs-lopt}. % \yunzhen{can we observe the fast and slow rate in the figure? Also we can comment more on the figure here, could be repeating some of the captions}

\textbf{Noisy Regime.}
Now fix $\sigma^2 \ne 0$ and $\phi_0 \in (0,1)$. In this regime, Theorem \ref{thm:darkseid} predicts that consistency (i.e. $E_{test}(\widehat w_n^{pred}) \overset{T \to \infty}{\longrightarrow} 0$) is only possible if $\ell \le \ell_\star$. First consider values of $\ell$ for which the clean variance $\sigma^2 T^{-(1-\ell/\beta)}$ is smaller than the clean bias $T^{-2\underline r\ell}$ in \eqref{eq:tron} %\julia{Sorry, what bias and variance are you talking about related to Thm. 5.2? Relate to the two expressions inside the max in thm 5.2},
i.e. $0 \le \ell \le \ell_{crit}$. We get
$$
E_{test}(\widehat w_n^{pred}) \asymp T^{-2\ell \underline r} + T^{-(b-a-\ell/\beta)},
$$
which is minimized by taking $\ell = \min(\ell_\star,\ell_{crit})$. For other values of $\ell$, the variance dominates and we have
\begin{eqnarray*}
    E_{test}(\widehat w_n^{pred}) \asymp T^{-(1-\ell/\beta)} + T^{-(b-\ell/\beta -a)} \asymp T^{-(\gamma-\ell/\beta)},
\end{eqnarray*}
where $\gamma := \min(1,b-a)$. This is minimized by taking $\ell = \ell_{crit}$, leading to $E_{test}(\widehat w_n^{pred}) \asymp T^{-(\gamma - 1/(2\beta\underline r + 1))}$. This tends to zero with $T \to \infty$ only if $b > a + 1/(2\beta\underline r + 1)$.

% \ElvisIssue{(TODO) Say something about neural networks. \citet{Cui_2022} managed to tell a story for some fabricated NN setups (MNIST regression, etc.). We could play a similar game here.}

% \ElvisIssue{Related picture for $(n,T_0,\sigma_0^2,\sigma_2,T)$ to the picture for $(0,\infty,0,\widetilde \sigma_2,\widetilde T)$ for some $\widetilde\sigma^2$ and $\widetilde T$ to be determined}

% \textcolor{blue}{\textit{\lipsum[23]}}

\section{Experiments}\label{sec:exp}
We perform the following experiments on both simulated and real data to empirically validate our theoretical results.

\subsection{Simulated Data}
We consider ordinary / linear ridge regression in $\mathbb R^d$, for $d=300$ and different structures for the covariance matrix $\Sigma$ of the inputs: isotropic (i.e $\Sigma=I_d$) and power-law \eqref{eq:powerlaw}, with $(\beta,r)=(2,0.375)$. For each value of $n$ (the generation index), the fake data-generator is constructed according to the process described in \eqref{eq:gptchain}. Then, for different values of $T$ (between $1$ and $1000,000$),  a sample of size $T$ is drawn from this fake data-generator and then a downstream ridge model \eqref{eq:ridge} is fitted. The test set consists of $100,000$ clean pairs $(x,y)$ form the true data distribution $P_{\Sigma,w_0,\sigma^2}$.
This experiment is repeated $10$ times to generate error bars. The results for the isotropic setting are shown in Figure \ref{fig:linreg} and the results for the power-law setting are shown in Figure \ref{fig:krreg}. Figure \ref{fig:harmsway} shows the over-parametrized setting.

\subsection{Real Data: Kernel Ridge Regression on MNIST}
As in \citet{Cui_2022,WeiMoreThanToy} we consider a distribution on MNIST, a popular dataset in the ML community. The classification dataset contains $60,000$ training and $10,000$ test data points (handwritten), with labels from 0 to 9 inclusive. Like in \citet{Cui_2022}, we convert the labels into real numbers (i.e a regression problem) as follows: $y = \text{label} \text{ mod }2 + \text{ noise }$, where the variance of the noise is $\sigma^2=1$ (for simplicity, we also set $\sigma_0^2=1$). The test set consists of $10,000$ pairs $(x,y)$, with the labels $y$ constructed as described in the previous sentence. The fake data used for training is generated as in the previous experiment, but via kernel ridge regression (instead of least squares) with the RBF kernel (bandwidth = $10^{-4}$) and the polynomial kernel (degree = $5$, bandwidth = $10^{-3}$). Note that it was empirically shown in \citet{Cui_2022} that these datasets verify \eqref{eq:powerlaw} with $(\beta,r) \approx (1.65,0.097)$ in the case of the aforementioned RBF kernel, and $(\beta,r) \approx (1.2,0.15)$ in the case of the polynomial kernel. Then, for different values of $T$ (between $1$ and $1000$),  a sample of size $T$ is drawn from this fake data-generator and then a downstream kernel ridge model is fitted. Each of these experiments are repeated $10$ times to generate error bars (due to different realizations of label noise). The results are shown in Figure \ref{fig:lcrit-vs-lopt}.

\section{Concluding Remarks}
%\textcolor{blue}{\textit{\lipsum[2-3]}}

As we navigate the "synthetic data age", our findings signal a departure from traditional test error rates (e.g neural scaling laws), introducing novel challenges and phenomena with the integration of synthetic data from preceding AI models into training sets.  Our work provides a solid analytical handle for demystifying the model collapse phenomenon as a modification of usual scaling laws caused by fake / synthesized training data.

A direct consequence of our multiplicative degradation result is that, over time (i.e as the number of generations becomes large), the effect of large language models (like ChatGPT) in the wild will be a pollution of the web to the extent that learning will be impossible. This will likely increase the value and cost of clean / non-AI-generated data. 

On the practical side, our analysis reveals that AI-generated data alters the optimal regularization for downstream models. Drawing from the insight that regularization mirrors early stopping \citep{pmlr-v89-ali19earlystop}, our study suggests that models trained on mixed real and AI-generated data may initially improve but later decline in performance (model collapse), necessitating early detection of this inflection point. This observation prompts a re-evaluation of current training approaches and underscores the complexity of model optimization in the era of synthetic data.

% \yunzhen{test}

\clearpage

\bibliography{sample.bib}

\begin{thebibliography}{52}
\providecommand{\natexlab}[1]{#1}
\providecommand{\url}[1]{\texttt{#1}}
\expandafter\ifx\csname urlstyle\endcsname\relax
  \providecommand{\doi}[1]{doi: #1}\else
  \providecommand{\doi}{doi: \begingroup \urlstyle{rm}\Url}\fi

\bibitem[Achiam et~al.(2023)Achiam, Adler, Agarwal, Ahmad, Akkaya, Aleman, Almeida, Altenschmidt, Altman, Anadkat, et~al.]{achiam2023gpt}
Achiam, J., Adler, S., Agarwal, S., Ahmad, L., Akkaya, I., Aleman, F.~L., Almeida, D., Altenschmidt, J., Altman, S., Anadkat, S., et~al.
\newblock Gpt-4 technical report.
\newblock \emph{arXiv preprint arXiv:2303.08774}, 2023.

\bibitem[Alemohammad et~al.(2023)Alemohammad, Casco-Rodriguez, Luzi, Humayun, Babaei, LeJeune, Siahkoohi, and Baraniuk]{alemohammad2023selfconsuming}
Alemohammad, S., Casco-Rodriguez, J., Luzi, L., Humayun, A.~I., Babaei, H., LeJeune, D., Siahkoohi, A., and Baraniuk, R.~G.
\newblock Self-consuming generative models go mad.
\newblock \emph{arXiv preprint arxiv:2307.01850}, 2023.

\bibitem[Ali et~al.(2019)Ali, Kolter, and Tibshirani]{pmlr-v89-ali19earlystop}
Ali, A., Kolter, J.~Z., and Tibshirani, R.~J.
\newblock A continuous-time view of early stopping for least squares regression.
\newblock In Chaudhuri, K. and Sugiyama, M. (eds.), \emph{Proceedings of the Twenty-Second International Conference on Artificial Intelligence and Statistics}, volume~89 of \emph{Proceedings of Machine Learning Research}, pp.\  1370--1378. PMLR, 16--18 Apr 2019.

\bibitem[Bach(2023)]{bach2023highdimensional}
Bach, F.
\newblock High-dimensional analysis of double descent for linear regression with random projections, 2023.

\bibitem[Bai \& Silverstein(2010)Bai and Silverstein]{BaiRMTBook}
Bai, Z. and Silverstein, J. W. J.~W.
\newblock \emph{Spectral analysis of large dimensional random matrices}.
\newblock Springer series in statistics. Springer, New York ;, 2nd ed. edition, 2010.
\newblock ISBN 9781441906601.

\bibitem[Belkin et~al.(2018)Belkin, Ma, and Mandal]{Belkin18understand}
Belkin, M., Ma, S., and Mandal, S.
\newblock To understand deep learning we need to understand kernel learning.
\newblock In \emph{Proceedings of the 35th International Conference on Machine Learning}, volume~80 of \emph{Proceedings of Machine Learning Research}, pp.\  541--549. PMLR, 2018.

\bibitem[Berthier et~al.(2020)Berthier, Bach, and Gaillard]{Berthier20}
Berthier, R., Bach, F.~R., and Gaillard, P.
\newblock Tight nonparametric convergence rates for stochastic gradient descent under the noiseless linear model.
\newblock \emph{CoRR}, abs/2006.08212, 2020.
\newblock URL \url{https://arxiv.org/abs/2006.08212}.

\bibitem[Bertrand et~al.(2023)Bertrand, Bose, Duplessis, Jiralerspong, and Gidel]{bertrand2023stability}
Bertrand, Q., Bose, A.~J., Duplessis, A., Jiralerspong, M., and Gidel, G.
\newblock On the stability of iterative retraining of generative models on their own data.
\newblock \emph{arXiv preprint arxiv:2310.00429}, 2023.

\bibitem[Bohacek \& Farid(2023)Bohacek and Farid]{bohacek2023nepotistically}
Bohacek, M. and Farid, H.
\newblock Nepotistically trained generative-ai models collapse, 2023.

\bibitem[Bordelon et~al.(2020)Bordelon, Canatar, and Pehlevan]{BordelonCP20spectrum}
Bordelon, B., Canatar, A., and Pehlevan, C.
\newblock Spectrum dependent learning curves in kernel regression and wide neural networks.
\newblock In \emph{Proceedings of the 37th International Conference on Machine Learning, {ICML} 2020, 13-18 July 2020, Virtual Event}, volume 119 of \emph{Proceedings of Machine Learning Research}, pp.\  1024--1034. {PMLR}, 2020.

\bibitem[Briesch et~al.(2023)Briesch, Sobania, and Rothlauf]{briesch2023large}
Briesch, M., Sobania, D., and Rothlauf, F.
\newblock Large language models suffer from their own output: An analysis of the self-consuming training loop, 2023.

\bibitem[Brown et~al.(2020)Brown, Mann, Ryder, Subbiah, Kaplan, Dhariwal, Neelakantan, Shyam, Sastry, Askell, Agarwal, Herbert-Voss, Krueger, Henighan, Child, Ramesh, Ziegler, Wu, Winter, Hesse, Chen, Sigler, Litwin, Gray, Chess, Clark, Berner, McCandlish, Radford, Sutskever, and Amodei]{NEURIPS2020_1457c0d6}
Brown, T., Mann, B., Ryder, N., Subbiah, M., Kaplan, J.~D., Dhariwal, P., Neelakantan, A., Shyam, P., Sastry, G., Askell, A., Agarwal, S., Herbert-Voss, A., Krueger, G., Henighan, T., Child, R., Ramesh, A., Ziegler, D., Wu, J., Winter, C., Hesse, C., Chen, M., Sigler, E., Litwin, M., Gray, S., Chess, B., Clark, J., Berner, C., McCandlish, S., Radford, A., Sutskever, I., and Amodei, D.
\newblock Language models are few-shot learners.
\newblock In Larochelle, H., Ranzato, M., Hadsell, R., Balcan, M., and Lin, H. (eds.), \emph{Advances in Neural Information Processing Systems}, volume~33, pp.\  1877--1901. Curran Associates, Inc., 2020.

\bibitem[Caponnetto \& de~Vito(2007)Caponnetto and de~Vito]{Caponnetto2007OptimalRF}
Caponnetto, A. and de~Vito, E.
\newblock Optimal rates for the regularized least-squares algorithm.
\newblock \emph{Foundations of Computational Mathematics}, 7:\penalty0 331--368, 2007.

\bibitem[Chizat et~al.(2019)Chizat, Oyallon, and Bach]{chizat}
Chizat, L., Oyallon, E., and Bach, F.
\newblock On lazy training in differentiable programming.
\newblock \emph{Advances in neural information processing systems}, 32, 2019.

\bibitem[Cui et~al.(2021)Cui, Loureiro, Krzakala, and Zdeborova]{cui2021generalization}
Cui, H., Loureiro, B., Krzakala, F., and Zdeborova, L.
\newblock Generalization error rates in kernel regression: The crossover from the noiseless to noisy regime.
\newblock In Beygelzimer, A., Dauphin, Y., Liang, P., and Vaughan, J.~W. (eds.), \emph{Advances in Neural Information Processing Systems}, 2021.

\bibitem[Cui et~al.(2022)Cui, Loureiro, Krzakala, and Zdeborová]{Cui_2022}
Cui, H., Loureiro, B., Krzakala, F., and Zdeborová, L.
\newblock Generalization error rates in kernel regression: the crossover from the noiseless to noisy regime.
\newblock \emph{Journal of Statistical Mechanics: Theory and Experiment}, 2022\penalty0 (11):\penalty0 114004, nov 2022.

\bibitem[Cui et~al.(2023)Cui, Loureiro, Krzakala, and Zdeborová]{Cui_2023}
Cui, H., Loureiro, B., Krzakala, F., and Zdeborová, L.
\newblock Error scaling laws for kernel classification under source and capacity conditions.
\newblock \emph{Machine Learning: Science and Technology}, 4\penalty0 (3):\penalty0 035033, August 2023.
\newblock ISSN 2632-2153.

\bibitem[Devlin et~al.(2018)Devlin, Chang, Lee, and Toutanova]{devlin2018bert}
Devlin, J., Chang, M.-W., Lee, K., and Toutanova, K.
\newblock Bert: Pre-training of deep bidirectional transformers for language understanding.
\newblock \emph{arXiv preprint arXiv:1810.04805}, 2018.

\bibitem[Dohmatob et~al.(2024)Dohmatob, Feng, Yang, Charton, and Kempe]{dohmatob2024tale}
Dohmatob, E., Feng, Y., Yang, P., Charton, F., and Kempe, J.
\newblock A tale of tails: Model collapse as a change of scaling laws, 2024.

\bibitem[Guo et~al.(2023)Guo, Shang, Vazirgiannis, and Clavel]{guo2023curious}
Guo, Y., Shang, G., Vazirgiannis, M., and Clavel, C.
\newblock The curious decline of linguistic diversity: Training language models on synthetic text, 2023.

\bibitem[Hastie et~al.(2022)Hastie, Montanari, Rosset, and Tibshirani]{Hastie2019Surprises}
Hastie, T., Montanari, A., Rosset, S., and Tibshirani, R.~J.
\newblock {Surprises in high-dimensional ridgeless least squares interpolation}.
\newblock \emph{The Annals of Statistics}, 50\penalty0 (2), 2022.

\bibitem[Hataya et~al.(2023)Hataya, Bao, and Arai]{Hataya_2023_ICCV}
Hataya, R., Bao, H., and Arai, H.
\newblock Will large-scale generative models corrupt future datasets?
\newblock In \emph{Proceedings of the IEEE/CVF International Conference on Computer Vision (ICCV)}, pp.\  20555--20565, October 2023.

\bibitem[Hoffmann et~al.(2022)Hoffmann, Borgeaud, Mensch, Buchatskaya, Cai, Rutherford, de~Las~Casas, Hendricks, Welbl, Clark, Hennigan, Noland, Millican, van~den Driessche, Damoc, Guy, Osindero, Simonyan, Elsen, Rae, Vinyals, and Sifre]{hoffmann2022trainingChinchilla}
Hoffmann, J., Borgeaud, S., Mensch, A., Buchatskaya, E., Cai, T., Rutherford, E., de~Las~Casas, D., Hendricks, L.~A., Welbl, J., Clark, A., Hennigan, T., Noland, E., Millican, K., van~den Driessche, G., Damoc, B., Guy, A., Osindero, S., Simonyan, K., Elsen, E., Rae, J.~W., Vinyals, O., and Sifre, L.
\newblock Training compute-optimal large language models, 2022.

\bibitem[Jacot et~al.(2018)Jacot, Gabriel, and Hongler]{NEURIPS2018_Jacot}
Jacot, A., Gabriel, F., and Hongler, C.
\newblock Neural tangent kernel: Convergence and generalization in neural networks.
\newblock In Bengio, S., Wallach, H., Larochelle, H., Grauman, K., Cesa-Bianchi, N., and Garnett, R. (eds.), \emph{Advances in Neural Information Processing Systems}, volume~31. Curran Associates, Inc., 2018.

\bibitem[Kaplan et~al.(2020)Kaplan, McCandlish, Henighan, Brown, Chess, Child, Gray, Radford, Wu, and Amodei]{kaplan2020scaling}
Kaplan, J., McCandlish, S., Henighan, T., Brown, T.~B., Chess, B., Child, R., Gray, S., Radford, A., Wu, J., and Amodei, D.
\newblock Scaling laws for neural language models.
\newblock \emph{arXiv preprint arXiv:2001.08361}, 2020.

\bibitem[Knowles \& Yin(2017)Knowles and Yin]{Knowles2017}
Knowles, A. and Yin, J.
\newblock Anisotropic local laws for random matrices.
\newblock \emph{Probability Theory and Related Fields}, 169\penalty0 (1):\penalty0 257--352, 2017.

\bibitem[Lee et~al.(2018)Lee, Bahri, Novak, Schoenholz, Pennington, and Sohl{-}Dickstein]{ICLR18_LeeNTK}
Lee, J., Bahri, Y., Novak, R., Schoenholz, S.~S., Pennington, J., and Sohl{-}Dickstein, J.
\newblock Deep neural networks as gaussian processes.
\newblock In \emph{6th International Conference on Learning Representations, {ICLR} 2018, Vancouver, BC, Canada, April 30 - May 3, 2018, Conference Track Proceedings}. OpenReview.net, 2018.

\bibitem[Liang \& Rakhlin(2020)Liang and Rakhlin]{JustInterpolate}
Liang, T. and Rakhlin, A.
\newblock {Just interpolate: Kernel “Ridgeless” regression can generalize}.
\newblock \emph{The Annals of Statistics}, 48\penalty0 (3), 2020.

\bibitem[Liu et~al.(2019)Liu, Ott, Goyal, Du, Joshi, Chen, Levy, Lewis, Zettlemoyer, and Stoyanov]{liu2019roberta}
Liu, Y., Ott, M., Goyal, N., Du, J., Joshi, M., Chen, D., Levy, O., Lewis, M., Zettlemoyer, L., and Stoyanov, V.
\newblock Roberta: A robustly optimized bert pretraining approach.
\newblock \emph{arXiv preprint arXiv:1907.11692}, 2019.

\bibitem[Maloney et~al.(2022)Maloney, Roberts, and Sully]{maloney2022solvable}
Maloney, A., Roberts, D.~A., and Sully, J.
\newblock A solvable model of neural scaling laws, 2022.

\bibitem[Martínez et~al.(2023{\natexlab{a}})Martínez, Watson, Reviriego, Hernández, Juarez, and Sarkar]{martínez2023combining}
Martínez, G., Watson, L., Reviriego, P., Hernández, J.~A., Juarez, M., and Sarkar, R.
\newblock Combining generative artificial intelligence (ai) and the internet: Heading towards evolution or degradation?
\newblock \emph{arXiv preprint arxiv: 2303.01255}, 2023{\natexlab{a}}.

\bibitem[Martínez et~al.(2023{\natexlab{b}})Martínez, Watson, Reviriego, Hernández, Juarez, and Sarkar]{martínez2023understanding}
Martínez, G., Watson, L., Reviriego, P., Hernández, J.~A., Juarez, M., and Sarkar, R.
\newblock Towards understanding the interplay of generative artificial intelligence and the internet.
\newblock \emph{arXiv preprint arxiv: 2306.06130}, 2023{\natexlab{b}}.

\bibitem[Marčenko \& Pastur(1967)Marčenko and Pastur]{MarcenkoPastur}
Marčenko, V. and Pastur, L.
\newblock Distribution of eigenvalues for some sets of random matrices.
\newblock \emph{Math USSR Sb}, 1:\penalty0 457--483, 01 1967.

\bibitem[Midjourney(2023)]{midjourney2023}
Midjourney.
\newblock Midjourney ai, 2023.
\newblock URL \url{https://www.midjourney.com/}.

\bibitem[Mobahi et~al.(2020)Mobahi, Farajtabar, and Bartlett]{bartlett2020neurips}
Mobahi, H., Farajtabar, M., and Bartlett, P.
\newblock Self-distillation amplifies regularization in {H}ilbert space.
\newblock In \emph{Advances in Neural Information Processing Systems}, volume~33, pp.\  3351--3361. Curran Associates, Inc., 2020.

\bibitem[Neal(1996)]{Neal1996PriorsFI}
Neal, R.~M.
\newblock Priors for infinite networks.
\newblock In \emph{Bayesian Learning for Neural Networks}, pp.\  29--53. Springer, New York, 1996.

\bibitem[Nitanda \& Suzuki(2021)Nitanda and Suzuki]{nitanda2021optimal}
Nitanda, A. and Suzuki, T.
\newblock Optimal rates for averaged stochastic gradient descent under neural tangent kernel regime.
\newblock In \emph{International Conference on Learning Representations}, 2021.

\bibitem[Pillaud{-}Vivien et~al.(2018)Pillaud{-}Vivien, Rudi, and Bach]{PillaudVivien18}
Pillaud{-}Vivien, L., Rudi, A., and Bach, F.~R.
\newblock Statistical optimality of stochastic gradient descent on hard learning problems through multiple passes.
\newblock In Bengio, S., Wallach, H.~M., Larochelle, H., Grauman, K., Cesa{-}Bianchi, N., and Garnett, R. (eds.), \emph{Advances in Neural Information Processing Systems 31: Annual Conference on Neural Information Processing Systems 2018}, pp.\  8125--8135, 2018.

\bibitem[Rahimi \& Recht(2008)Rahimi and Recht]{RahimiRecht2008}
Rahimi, A. and Recht, B.
\newblock Weighted sums of random kitchen sinks: Replacing minimization with randomization in learning.
\newblock In \emph{Advances in Neural Information Processing Systems}. Curran Associates, Inc., 2008.

\bibitem[Ramesh et~al.(2021)Ramesh, Pavlov, Goh, Gray, Voss, Radford, Chen, and Sutskever]{pmlr-v139-ramesh21a}
Ramesh, A., Pavlov, M., Goh, G., Gray, S., Voss, C., Radford, A., Chen, M., and Sutskever, I.
\newblock Zero-shot text-to-image generation.
\newblock In Meila, M. and Zhang, T. (eds.), \emph{Proceedings of the 38th International Conference on Machine Learning}, volume 139 of \emph{Proceedings of Machine Learning Research}, pp.\  8821--8831. PMLR, 18--24 Jul 2021.

\bibitem[Richards et~al.(2021)Richards, Mourtada, and Rosasco]{Richards2021}
Richards, D., Mourtada, J., and Rosasco, L.
\newblock Asymptotics of ridge(less) regression under general source condition.
\newblock In \emph{Proceedings of The 24th International Conference on Artificial Intelligence and Statistics}, volume 130 of \emph{Proceedings of Machine Learning Research}. PMLR, 2021.

\bibitem[Rombach et~al.(2022)Rombach, Blattmann, Lorenz, Esser, and Ommer]{Rombach_2022_CVPR}
Rombach, R., Blattmann, A., Lorenz, D., Esser, P., and Ommer, B.
\newblock High-resolution image synthesis with latent diffusion models.
\newblock In \emph{Proceedings of the IEEE/CVF Conference on Computer Vision and Pattern Recognition (CVPR)}, pp.\  10684--10695, June 2022.

\bibitem[Rudi \& Rosasco(2017)Rudi and Rosasco]{Rudy2017RF}
Rudi, A. and Rosasco, L.
\newblock Generalization properties of learning with random features.
\newblock In \emph{Advances in Neural Information Processing Systems}. Curran Associates Inc., 2017.
\newblock ISBN 9781510860964.

\bibitem[Schmidt-Hieber(2017)]{SchmidtHieber2017scaling}
Schmidt-Hieber, J.
\newblock Nonparametric regression using deep neural networks with relu activation function.
\newblock \emph{Annals of Statistics}, 48, 08 2017.

\bibitem[Shumailov et~al.(2023)Shumailov, Shumaylov, Zhao, Gal, Papernot, and Anderson]{shumailov2023curse}
Shumailov, I., Shumaylov, Z., Zhao, Y., Gal, Y., Papernot, N., and Anderson, R.
\newblock The curse of recursion: Training on generated data makes models forget.
\newblock \emph{arXiv preprint arxiv:2305.17493}, 2023.

\bibitem[Simon et~al.(2021)Simon, Dickens, and DeWeese]{SimonNTK2021}
Simon, J.~B., Dickens, M., and DeWeese, M.~R.
\newblock Neural tangent kernel eigenvalues accurately predict generalization.
\newblock 2021.

\bibitem[Spigler et~al.(2020)Spigler, Geiger, and Wyart]{Spigler_2020}
Spigler, S., Geiger, M., and Wyart, M.
\newblock Asymptotic learning curves of kernel methods: empirical data versus teacher–student paradigm.
\newblock \emph{Journal of Statistical Mechanics: Theory and Experiment}, 2020\penalty0 (12), 2020.

\bibitem[Suzuki(2019)]{suzuki2018adaptivity}
Suzuki, T.
\newblock Adaptivity of deep re{LU} network for learning in besov and mixed smooth besov spaces: optimal rate and curse of dimensionality.
\newblock In \emph{International Conference on Learning Representations}, 2019.

\bibitem[Taori \& Hashimoto(2023)Taori and Hashimoto]{Taoridatafeedback23}
Taori, R. and Hashimoto, T.~B.
\newblock Data feedback loops: model-driven amplification of dataset biases.
\newblock ICML'23. JMLR.org, 2023.

\bibitem[Touvron et~al.(2023)Touvron, Martin, Stone, Albert, Almahairi, Babaei, Bashlykov, Batra, Bhargava, Bhosale, et~al.]{touvron2023llama}
Touvron, H., Martin, L., Stone, K., Albert, P., Almahairi, A., Babaei, Y., Bashlykov, N., Batra, S., Bhargava, P., Bhosale, S., et~al.
\newblock Llama 2: Open foundation and fine-tuned chat models.
\newblock \emph{arXiv preprint arXiv:2307.09288}, 2023.

\bibitem[Wei et~al.(2022)Wei, Hu, and Steinhardt]{WeiMoreThanToy}
Wei, A., Hu, W., and Steinhardt, J.
\newblock More than a toy: Random matrix models predict how real-world neural representations generalize.
\newblock In \emph{Proceedings of the 39th International Conference on Machine Learning}, volume 162 of \emph{Proceedings of Machine Learning Research}. PMLR, 2022.

\bibitem[Williams(1996)]{NIPS1996_InfiniteNN}
Williams, C.
\newblock Computing with infinite networks.
\newblock In Mozer, M., Jordan, M., and Petsche, T. (eds.), \emph{Advances in Neural Information Processing Systems}, volume~9. MIT Press, 1996.

\end{thebibliography}
\bibliographystyle{icml2024}

%%%%%%%%%%%%%%%%%%%%%%%%%%%%%%%%%%%%%%%%%%%%%%%%%%%%%%%%%%%%%%%%%%%%%%%%%%%%%%%
%%%%%%%%%%%%%%%%%%%%%%%%%%%%%%%%%%%%%%%%%%%%%%%%%%%%%%%%%%%%%%%%%%%%%%%%%%%%%%%
% APPENDIX
%%%%%%%%%%%%%%%%%%%%%%%%%%%%%%%%%%%%%%%%%%%%%%%%%%%%%%%%%%%%%%%%%%%%%%%%%%%%%%%
%%%%%%%%%%%%%%%%%%%%%%%%%%%%%%%%%%%%%%%%%%%%%%%%%%%%%%%%%%%%%%%%%%%%%%%%%%%%%%%
\newpage
\appendix
\onecolumn

\addtocontents{toc}{\protect\setcounter{tocdepth}{2}}
\begin{center}
    \noindent\rule{17cm}{3pt} \vspace{0.4cm}
    
    \Large \textbf{Appendix / Supplementary Material for} \\ ~\\[-0.5cm]
   \Large \textbf{Model Collapse Demystified: The Case of Regression} 
    
    \noindent\rule{17cm}{1.2pt}
\end{center}
\tableofcontents

% \addtocontents{toc}{\protect\setcounter{tocdepth}{2}}
% \begin{center}
%     \noindent\rule{17cm}{3pt} \vspace{0.4cm}
    
%     \huge \textbf{Appendix} \\ ~\\[-0.5cm]
%     \Large \textbf{Model Collapse Demystified: The Case of Regression} 
    
%     \noindent\rule{17cm}{1.2pt}
% \end{center}

\section{Exact Characterization of Test Error Under Model Collapse}\label{app:A}

\subsection{Proof of Theorem \ref{thm:linreg} (Rigeless Regression)}
The proof is by induction on the number of generations $n$ of fake data. For $n=0$,  we have
\begin{eqnarray}
\begin{aligned}
    E_{test}(\widehat w_0^{pred}) &=  \mathbb E\, \|\widehat w_0^{pred} - w_0\|^2_\Sigma = \mathbb E\, \|\widehat w_0^{pred} - \widehat w_0\|^2_2 = \mathbb E \|(X_0^\top X_0)^{-1} X_0^\top E_{0}\|^2_2\\
    &= \sigma^2 \mathbb E \trace(X_0^\top X_{0})^{-1} = \sigma^2 \frac{d}{T-d-1} \simeq \frac{\sigma^2\phi }{1-\phi},
\end{aligned}
 \end{eqnarray}
 where $\phi := d/T \in (0,1)$ and the last step has made use of Lemma \ref{lm:expinv} below. This is a well-known result for the test error of linear regression in the under-parametrized regime, without any AI pollution (fake / synthesized training data).

Analogously, for $n=1$ one computes the test error after the first generation of fake data as follows
\begin{eqnarray*}
\begin{aligned}
    E_{test}(\widehat w_1^{pred}) &=  \mathbb E\|\widehat w_1^{pred} - w_0\|^2_\Sigma = \mathbb E\|\widehat w_1^{pred} - \widehat w_0\|^2_2 = \mathbb E\|\widehat w_1^{pred} - \widehat w_1 + \widehat w_1 - \widehat w_0\|^2_2 \\
    &= \mathbb E\|(X_0^\top X_0)^{-1} X_0 E_0 + \widehat w_0^{pred} - w_0\|^2_2 = \mathbb E\, \|w_0-\widehat w_0^{pred}\|^2_2 + \mathbb E\,\|(X_0^\top X_0)^{-1} X_0^\top E_0\|^2_2\\
    &= E_{test}(\widehat w_0^{pred}) + \frac{\sigma_0^2d}{T_0-d-1}\simeq \frac{\sigma^2\phi}{1-\phi} + \frac{\sigma_0^2\phi_0}{1-\phi_0},
    \end{aligned}
\end{eqnarray*}
where $\phi_0 = d/T_0 \in (0,1)$. 
Continuing the induction on $n$, we obtain the result. \qed

\begin{lemma}
    Let $X_0$ be an $T_0 \times d$ random matrix with iid rows from $N(0,\Sigma)$. If $T_0 \ge d + 2$, then the empirical covariance matrix $\widehat \Sigma_0 := X_0^\top X_0/T_0$ is invertible a.s and
    $$
    \mathbb E\,[\widehat \Sigma_0^{-1}] = \dfrac{T_0}{T_0-d-1}\Sigma^{-1}.
    $$
    \label{lm:expinv}
\end{lemma}

\subsection{Proof of Theorem \ref{thm:generalcov} (Ridge Regression + General Covariance)}

\subsubsection{Representation of $\widehat w_n$ and $\widehat w_n^{pred}$}
We first obtain explicit formulae for the labelling vectors $\widehat w_n$ used in the fake-data generation process \eqref{eq:fakew}.
For any integer $m \ge 0$, define  $P_m = X_m^\dagger X_m$, the orthogonal projection matrix onto the subspace of $\mathbb R^d$ spanned by the rows of $X_m$. Observe from \eqref{eq:fakew} that
\begin{eqnarray}
\begin{split}
\widehat w_n &= X_{n-1}^\dagger \overline Y_{n-1} = X_{n-1}^\dagger (X_{n-1} \widehat w_{n-1}  + E_{n-1}) = P_{n-1} \widehat w_{n-1} + X_{n-1}^\dagger E_{n-1}\\
&= P_{n-1} X_{n-2}^\dagger(X_{n-2}\widehat w_{n-2} + E_{n-2}) + X_{n-1}^\dagger E_{n-1}\\
&= P_{n-1}P_{n-2}\widehat w_{n-2} + P_{n-1}X_{n-2}^\dagger E_{n-2} + X_{n-1}^\dagger E_{n-1}\\
&\quad \vdots\\
&= P_{n-1}P_{n-2}\ldots P_0 w_0 + P_{n-1} P_{n-2}\ldots P_1 X_1^\dagger E_1 + P_{n-1}P_{n-2}\ldots P_2 X_2^\dagger E_2 + \ldots\\
&\quad \vdots\\
&= P_{n-1}P_{n-2}\ldots P_0 w_0 + \sum_{m=0}^{n-1} P_{n-1}P_{n-2}\ldots P_m X_m^\dagger E_m.
\end{split}
\end{eqnarray}
We get the following result.
\begin{lemma}
For any $n \ge 0$, the following formula holds
\begin{eqnarray}
\widehat w_n = \begin{cases}
w_0,&\mbox{ if }n=0,\\
Q_{n-1} w_0 + \sum_{m=0}^{n-1} \overline Q_{n-1,m} E_m,&\mbox{ if }n \ge 1,
\end{cases}
\end{eqnarray}
\label{lm:whatrep}
where $\overline Q_{k,m} := Q_{k,m}X_m^\dagger$, $Q_{k,m} := P_kP_{k-1}\ldots P_m$ and $Q_k := Q_{k,0}=P_k P_{k-1} \ldots P_0$.
Moreover, $\widehat w_n \in \operatorname{Im}P_{n-1}$ as soon as $n \ge 1$.

In particular, under the simplifying condition \eqref{eq:simplifier}, it holds that
\begin{eqnarray}
\widehat w_n = \begin{cases}
w_0,&\mbox{ if }n=0,\\
P_0 w_0 + X_0^\dagger \overline E_{n-1} \in \operatorname{Im} P_0,&\mbox{ if }n \ge 1.
\end{cases}
\label{eq:whatrepsimplified}
\end{eqnarray}
where $\overline E_{n-1} := \sum_{m=0}^{n-1} E_m$, a random vector of length $T_0$, with iid entries from $N(0,\textcolor{red}{n\sigma_0^2})$, and independent of $X_0$.
Moreover, $\widehat w_n \in \operatorname{Im}P_0$ as soon as $n \ge 1$.
\end{lemma}
Note that the second part of the result uses the elementary linear-algebraic fact that $P_m X_m^\dagger = X_m^\dagger$. In the special case where $T_0 \ge d$, we have $P_0 = I$ a.s., and so $\widehat w_n = w_0 + X_0^\dagger \overline E_{n-1}$. Otherwise, even in the absence of generator noise ($\sigma_0=0$), the fake data labeller $\widehat w_n=P_0w_0$ drifts away from the truth $w_0$, into a subspace of $\mathbb R^d$ spanned by the rows of $X_0$.
% , as in the previous version.

Next, let us obtain a decomposition for the downstream predictor $\widehat w_n^{pred}$ defined in \eqref{eq:ridge}. As usual, let $\widehat\Sigma := X^\top X/T$ be the empirical covariance matrix with resolvent $R=(\widehat \Sigma +\lambda I)^{-1}$, and observe that the downstream model writes
\begin{eqnarray}
\begin{split}
\widehat w_n^{pred} &= RX^\top \overline Y_n(X)/T = RX^\top (X\widehat w_n + E)/T\\
&= RX^\top (XQ_{n-1} w_0 + X\sum_{m=0}^{n-1}\overline Q_{n-1,m} E_m + E)/T\\
&= R\widehat \Sigma Q_{n-1} w_0 + RX^\top E/T + R \widehat \Sigma \sum_{m=0}^{n-1}\overline Q_{n-1,m} E_m.
\end{split}
\label{eq:whatdecompose}
\end{eqnarray}

\subsubsection{Proof of Theorem \ref{thm:generalcov}}
Using the decomposition \eqref{eq:whatdecompose} for the downstream model $\widehat w_n^{pred}$, we deduce that
\begin{eqnarray}
\begin{split}
E_{test}(\widehat w_n^{pred})  &= \mathbb E\,\|\widehat w_n^{pred} - w_0\|_\Sigma^2 = \mathbb E\,\|R \widehat \Sigma P_0 w_0 + RX^\top E/T + R \widehat \Sigma \sum_{m=0}^{n-1}\overline Q_{n-1,m} E_m -w_0\|_\Sigma^2\\
&= \mathbb E\,\|R \widehat \Sigma P_0 w_0 -w_0 + RX^\top E/T + R \widehat \Sigma \sum_{m=0}^{n-1}\overline Q_{n-1,m} E_m\|_\Sigma^2\\
&= \mathbb E\|R \widehat \Sigma P_0w_0-w_0\|_\Sigma^2 + \mathbb E\,\|RX^\top E/T\|_\Sigma^2 + \mathbb E\,\|R \widehat \Sigma \sum_{m=0}^{n-1}\overline Q_{n-1,m} E_m\|_\Sigma^2\\
&= \widetilde{Bias} + Var + n\sigma_0^2\rho,
\end{split}
\end{eqnarray}
where $\widehat \Sigma := X^\top X/T$, $\widetilde{Bias}$, $Var$, and $\rho$ are as given in the theorem. On the second line, we have used independence (of $X$, $X_0$, $E$, and $\overline E_{n-1}$) and the fact that $E$ and $\overline E_{n-1}$ are centered Gaussian random vectors, with iid components of variances $\sigma^2$ and $n\sigma_0^2$ respectively. \qed 

\subsection{Proof of Theorem \ref{thm:rho}}

\paragraph{Analysis of Bias-like Term.}
An exact analysis of the $\widetilde{Bias}$ term appearing in Theorems \ref{thm:generalcov} and \ref{thm:rho} is presumably a treacherous enterprise  given dependency on $X$ (via $R$ and $\widehat \Sigma$) and $X_0$ (via $P_0$).  In place of such an analysis, we shall settle for the following result which gives an instructive lower-bound.
\begin{proposition}
In the RMT limit \eqref{eq:nonproportionate}, it holds that $$
\lim \widetilde{Bias} - \lim Bias \ge \lim \mathbb E\,\|R \widehat \Sigma P_0 w_0-R \widehat \Sigma w_0\|_\Sigma^2 \ge 0.
$$
Thus, training on fake / synthesized data increases the bias term of the downstream model's test error !
\end{proposition}
\begin{proof}
Letting $A := R \widehat \Sigma$, one computes
\begin{eqnarray}
\begin{split}
\widetilde{Bias}-Bias &= \|AP_0w_0-w_0\|_\Sigma^2-\|Aw_0-w_0\|_\Sigma^2\\ &= \|AP_0w_0-Aw_0 + Aw-w\|_\Sigma^2-\|Aw_0-w_0\|_\Sigma^2\\
&= \|AP_0w_0-Aw_0\|_\Sigma^2 + 2w_0^\top (A-P_0 A)\Sigma (I-A)w_0\\
&= \|AP_0w_0-Aw_0\|_\Sigma ^2 + 2w_0^\top (I-P_0)A\Sigma (I-A)w_0.
\end{split}
\end{eqnarray}
It then suffices to observe that, in the RMT limit \eqref{eq:nonproportionate}, it holds that
$$
\lim\mathbb E\, w_0^\top (I-P_0)A\Sigma (I-A) w_0 \ge 0,
$$
as can be seen from repeated application of Propositions 1 and 2 of \citet{bach2023highdimensional}.
\end{proof}

\paragraph{Analysis of $\rho$ Term.}
Define a $d \times d$ random psd matrix $H := \widehat \Sigma R \Sigma \widehat \Sigma R$. Under the simplifying assumption \eqref{eq:simplifier}, the matrices $\overline Q_{k,m}$ defined in the theorem all equal $\overline Q_{0,0} = X_0^\dagger$. It follows that the $\rho$-term in \eqref{eq:b2} then writes
\begin{eqnarray}
\rho = \frac{1}{n}\sum_{m=0}^{n-1}\mathbb E\,[\overline Q_{n-1,m}\overline Q_{n-1,m}^\top H] = \mathbb E\,[\trace X_0^\dagger (X_0^\dagger)^\top H] = \mathbb E_H\mathbb E\,[\trace X_0^\dagger (X_0^\dagger)^\top H \mid H].
\label{eq:prerho}
\end{eqnarray}
Now, one computes the conditional expectation as follows
\begin{eqnarray}
    \begin{aligned}
\mathbb E\,[\trace X_0^\dagger (X_0^\dagger)^\top H \mid H] &= \mathbb E\,[\trace X_0^\top (X_0X_0^\top)^{-2}X_0 H \mid H]\\
&= \lim_{\lambda_0 \to 0^+}\frac{1}{T_0}\frac{\partial }{\partial \lambda_0} \mathbb E\,[\trace X_0^\top (X_0X_0^\top + \lambda_0 T_0 I)^{-1}X_0 H \mid H].       
    \end{aligned}
\end{eqnarray}
Furthermore, defining $A := \Sigma^{1/2} H \Sigma^{1/2}$ and $Z_0 = X_0\Sigma^{-1/2}$, we have 
\begin{eqnarray*}
    \begin{aligned}
\trace X_0^\top (X_0X_0^\top + \lambda_0 T_0 I)^{-1}X_0 H &= \trace \Sigma^{1/2}Z_0^\top (Z_0 \Sigma Z_0^\top + \lambda_0 T_0 I)^{-1} Z_0\Sigma^{1/2} H\\
&= \trace A Z_0^\top (Z_0 \Sigma Z_0^\top + \lambda_0 T_0 I)^{-1} Z_0,
    \end{aligned}
\end{eqnarray*}
We deduce from Proposition 2 of \citet{bach2023highdimensional} that
\begin{eqnarray}
    \mathbb E\,[\trace X_0^\top (X_0X_0^\top + \lambda_0 T_0 I)^{-1}X_0 H \mid H] \simeq \trace A(\Sigma + \kappa(\lambda_0,T_0) I)^{-1} = \trace H (\Sigma + \kappa(\lambda_0,T_0) I)^{-1}\Sigma.
\end{eqnarray}
Differentiating w.r.t. $\lambda_0$ and letting this parameter tend to zero from above gives
\begin{eqnarray}
\begin{aligned}
    \mathbb E\,[\trace X_0^\dagger (X_0^\dagger)^\top H \mid H] &= -\frac{1}{T_0}\lim_{\lambda_0 \to 0^+}\frac{\partial }{\partial \lambda_0}\mathbb E\,[\trace X_0^\top (X_0X_0^\top + \lambda_0 T_0 I)^{-1}X_0 H \mid H]\\
    &\simeq -\frac{1}{T_0}\lim_{\lambda_0 \to 0^+}\frac{\partial \kappa(\lambda_0,T_0)}{\partial \lambda_0}\cdot \frac{\partial }{\partial t}\trace H (\Sigma + t I)^{-1}\Sigma\bigg|_{t=\kappa(\lambda_0,T_0)} \simeq \frac{\trace H(\Sigma + \kappa_0 I)^{-2}\Sigma}{T_0 - \df_2(\kappa_0)},
\end{aligned}
\end{eqnarray}
where $\kappa_0 = \kappa(0,T_0)$, % , and verifies $\df_1(\kappa_0) = T_0$,
and we have made use of Lemma \ref{lm:bach}. Combining with \eqref{eq:prerho} and then applying Proposition 1 of \citet{bach2023highdimensional} to compute $\mathbb E_H\,\trace H(\Sigma + \kappa_0 I)^{-2}\Sigma = \mathbb E_X\,\trace \widehat \Sigma R \Sigma \widehat \Sigma R (\Sigma + \kappa_0 I)^{-2}\Sigma$ gives the following result.
\begin{proposition}
In the RMT limit \eqref{eq:nonproportionate}, it holds for any $\lambda > 0$ that
    \begin{eqnarray}
    \begin{aligned}
        \rho &=
        \dfrac{\trace \Sigma^4 (\Sigma + \kappa_0 I)^{-2} (\Sigma + \kappa I)^{-2}}{T_0 - \df_2(\kappa_0)} + \dfrac{\kappa^2 \trace\Sigma^2(\Sigma + \kappa_0 I)^{-2}(\Sigma + \kappa I)^{-2}}{T_0-\df_2(\kappa_0)}\cdot\dfrac{\df_2(\kappa)}{T-\df_2(\kappa)},
    \end{aligned}
    \end{eqnarray}
where $\kappa_0 := \kappa(\lambda_0,T_0)$ and $\kappa=\kappa(\lambda,T)$.

In particular, if $T_0 \ge d$, then
\begin{eqnarray}
\rho \simeq \frac{\operatorname{df}_2(\kappa)}{T-\operatorname{df}_2(\kappa)}\left(1 + \frac{ \kappa^2\trace(\Sigma + \kappa I)^{-2}}{T_0-\operatorname{df}_2(\kappa_0)}\right).
\end{eqnarray}
\end{proposition}
This result completes the proof of Theorem \ref{thm:rho}. \qed

\subsection{Proof of Corollary \ref{cor:previous-main-thm}}
% We now prove the corollary.
For the first part, we  know from Theorem \ref{thm:generalcov} that
    \begin{align}
        &E_{test}(\widehat w_n^{pred}) = E_{test}(\widehat w_0^{pred}) + \textcolor{red}{n \sigma_0^2  \rho},\text{ with}\\
        &\quad \rho := \frac{\mathbb E\, \trace\Sigma^{-1} \widehat \Sigma (\widehat \Sigma + \lambda I)^{-1} \Sigma (\widehat \Sigma + \lambda I)^{-1} \widehat \Sigma}{T_0-d}.
    \end{align}
    The $E_{test}(\widehat w_0^{pred})$ term is taken care of by Proposition \ref{prop:classicalrige}, since this corresponds to generalization error on clean training data. For the $\rho$ term, we use Proposition 1 of
     \citet{bach2023highdimensional} with $A=\Sigma^{-1}$ and $B=\Sigma$ to get
     \begin{eqnarray*}
     \begin{aligned}
         \rho &\simeq \frac{\trace(\Sigma + \kappa I)^{-2}\Sigma^2}{T_0-d} +  \frac{\kappa^2\trace(\Sigma + \kappa I)^{-2})}{T_0-d}\frac{\trace(\Sigma + \kappa I)^{-2}\Sigma^2}{T-\mathrm{df}_2(\kappa)}\\
         &=\frac{\mathrm{df}_2(\kappa)}{T_0-d} + \frac{\kappa^2\trace(\Sigma + \kappa I)^{-2}}{T_0-d}\frac{\mathrm{df}_2(\kappa)}{T-\mathrm{df}_2(\kappa)},
         \end{aligned}
     \end{eqnarray*}
which proves the first part of the result.

For the second part, note that  $\mathrm{df}_2(\kappa) = d/(1+\kappa)^2$ when $\Sigma=I$, \eqref{eq:MP} holds, and so
\begin{eqnarray*}
\begin{aligned}
    (1-1/\phi_0) \rho &\simeq \frac{1}{(1+\kappa)^2} + \frac{\kappa^2}{(1+\kappa)^4}\frac{d}{T-d/(1+\kappa)^2} \\
    &\simeq \frac{1}{(1+\kappa)^2} + \frac{\kappa^2}{(1+\kappa)^4}\frac{\phi}{1-\phi/(1+\kappa)^2}\\
    &= \frac{1}{(1+\kappa)^2} + \frac{1}{(1+\kappa)^2}\frac{\phi\kappa^2}{(1+\kappa)^2-\phi},
    \end{aligned}
\end{eqnarray*}
and the result follows. \qed

\subsection{A Note on Proposition \ref{prop:classicalrige}}
As mentioned in the main text, the result is classical \citet{Richards2021,Hastie2019Surprises,bach2023highdimensional}). Only the second part needs a comment which we now provide. Indeed, the second part of the result follows from the first as we now see.
Indeed, $w_0^\top \Sigma(\Sigma + \kappa I)^{-2} w_0 = r^2/(1+\kappa)^2$, $\mathrm{df}_2(\kappa) = d/(1+\kappa)^2$ and so we deduce from the first part that
\begin{align*}
Var &\simeq \sigma^2 \phi \frac{1}{(1+\kappa)^2}\frac{1}{1-\phi/(1+\kappa)^2} =  \frac{\sigma^2 \phi}{(1+\kappa)^2 - \phi},\\
Bias &\simeq \kappa^2 \|w_0\|_2^2 \frac{1}{(1+\kappa)^2} \frac{1}{1-\phi/(1+\kappa)^2} = \frac{\kappa^2 \|w_0\|_2^2}{(1+\kappa)^2 - \phi},
\end{align*}
from which the result follows. \qed

% \subsection{A Calculation}
% Let's compute $\Delta Bias = \mathbb E\,\|RS \delta\|_\Sigma^2$, where $\delta := Q_{n-1}w_0-w_0$. Setting $A = \delta\delta^\top$ and $B=\Sigma$, we have
% \begin{align*}
%     \Delta Bias &= \mathbb E\,\|RS \delta\|_\Sigma^2 = \mathbb E\, \delta^\top S R\Sigma RS \delta = \mathbb E\,\trace AS R B R S\\
%     &\simeq \trace A\Sigma(\Sigma + \kappa I)^{-1}B(\Sigma + \kappa I_d)^{-1} + \kappa^2\trace A(\Sigma + \kappa I)^{-2}\frac{\trace B(\Sigma + \kappa I)^{-2}}{T-\df_2(\kappa)}\\
%     &= \delta^\top \Sigma^2(\Sigma + \kappa I)^{-2}\delta + \kappa^2\delta^\top(\Sigma + \kappa I)^{-2}\delta\frac{\trace\Sigma(\Sigma + \kappa I)^{-2}}{T-\df_2(\kappa)},
% \end{align*}
% where $\kappa = \kappa(\lambda,T)$ is as defined in in \eqref{eq:kappa}, and the second line is thanks to Proposition 1 of \citet{bach2023highdimensional}.

We now  need to estimate $\delta^\top H \delta$ for a deterministic psd matrix $H$. Observe that
\begin{eqnarray}
    \begin{split}
\delta^\top H \delta &= (Q_{n-1}w_0-w_0)^\top H (Q_{n-1} w_0 - w_0)\\
&= w_0^\top Q_{n-1}^\top H Q_{n-1}w_0 -  2w_0^\top Q_{n-1}^\top H w_0  + w_0^\top H w_0. 
    \end{split}
\end{eqnarray}

\subsection{Proof of Theorem \ref{thm:P0} and Theorem \ref{thm:cata} (Model Collapse in the Absence of Label Noise)}

We first prove Theorem \ref{thm:cata}. Note that since we are in the isotropic case, $\Delta Bias$ defined in \eqref{eq:deltabias} is now given by $\Delta Bias := \mathbb E\,\|\widehat\Sigma R(Q_{n-1}w_0-w_0)\|^2$, where $Q_{n-1} := P_{n-1}P_{n-1}\ldots P_0$. Moreover, since $T>d$ and $\lambda = 0$ by assumption, we have $\Sigma R = I_d$, and so we further have $\Delta Bias := \mathbb E\,\|Q_{n-1}w_0-w_0\|^2$. Now, one computes
\begin{eqnarray}
\begin{split}
\|Q_{n-1}w_0 - w_0\|^2 &= \|w_0\|^2 - 2w_0^\top Q_{n-1}w_0 + w_0^\top Q_{n-1}^\top Q_{n-1} w_0\\
&= \|w_0\|^2-w_0^\top Q_{n-1} w_0\\
&\simeq \|w_0\|^2 - w_0^\top \left(\prod_{m=0}^{n-1}(I + \kappa_m I)^{-1}\right)w_0\\
&= \|w_0\|^2 - \|w_0\|^2\prod_{m=0}^{n-1}\min(1/\phi_m,1),
\end{split}
\end{eqnarray}
where on the 2nd line we have used the fact that $Q_{n-1}^\top Q_{n-1} = Q_{n-1}$ because the $P_m$'s are projections; on the 3rd line we have used Lemma \ref{lm:product} with $\Sigma=I$ and $u=v=w_0$; on the 4th line we have used the fact that $\kappa_m := \kappa(0,T_m) = \max(\phi_m-1,0) = \max(\phi_m,1)-1$. This completes the proof of Theorem \ref{thm:cata}.

The proof of Theorem \ref{thm:P0} is completely analogous, with $Q_{n-1}$ replaced with $Q_0$. \qed

\begin{lemma}
\label{lm:product}
    Let $X_0,\ldots,X_{n-1}$ be independent random matrices of shapes $T_m \times d$ for $m=0,\ldots,n-1$, with rows iid from $N(0,\Sigma)$, and let $Q_{n-1}:=P_{n-1}P_{n-2}\ldots P_0$, where $P_m = X_m^\dagger X_m$ is the orthogonal projection onto the subspace of $\mathbb R^d$ spanned by the rows of $X_m$. Then, in the limit $d,T_0,\ldots,T_{n-1} \to \infty$ such that $d/T_0 \to \phi_0 \in (0,\infty),\ldots,d/T_{n-1}\to \phi_{n-1} \in (0,\infty)$ with $\|\Sigma\|_{op},\|\Sigma^{-1}\|_{op} =O(1)$, it holds that for deterministic $L_2$-bounded sequences of vectors $u$ and $v$ 
    \begin{eqnarray}
        u^\top Q_{n-1} v \simeq u^\top \left(\prod_{m=0}^{n-1}(\Sigma + \kappa_m I)^{-1}\right)v,
    \end{eqnarray}
    where $\kappa_m=\kappa(0,T_m)$ is as defined in \eqref{eq:kappa}.
\end{lemma}
\begin{proof}
    The prof is by induction on $n \ge 1$. For $n=1$, we have $Q_{n-1} = Q_0 = P_0$. Thus,
    \begin{eqnarray}
    \begin{split}
    u^\top Q_0 v &= u^\top P_0 v = \lim_{t \to 0^+}u^\top X_0^\top (X_0X_0^\top + t I)^{-1}X_0 v\\
    &\simeq \lim_{t \to 0^+} u^\top (\Sigma + \kappa(t,T) I)^{-1}v= u^\top (\Sigma + \kappa_0 I)^{-1}v,
    \end{split}   
    \end{eqnarray}
    where $\kappa_0 := \kappa(0,T)$ and we used Proposition 2 of \citet{bach2023highdimensional} at the beginning of the 2nd line. Now, suppose the claim holds for $n$, and let's prove that it holds for $n+1$. Indeed,
    $$
    u^\top Q_n v = u^\top P_n Q_{n-1} v \simeq u^\top P_{n-1}\prod_{m=0}^{n-1}(\Sigma + \kappa_m)^{-1} v \simeq u^\top  \prod_{m=0}^n(\Sigma + \kappa_m)^{-1}v,
    $$
    where the second step is an application of the induction hypothesis with $P_n u$ in place of $u$.
\end{proof}

The following lemma can be used to compute $\|Q_{n-1}w_0 - w_0\|_\Sigma^2$ in the case of anisotropic $\Sigma$.
\begin{lemma}
    Under the hypothesis of Lemma \ref{lm:product}, it holds that
    \begin{align}
            u^\top Q_{n-1} v &\simeq u^\top \Sigma^n \left(\prod_{m=0}^{n-1}(\Sigma + \kappa_m I)^{-1}\right)v,\\
        u^\top Q_{n-1}^\top \Sigma Q_{n-1} v &\simeq u^\top \Sigma^n\left(\prod_{m=0}^{n-1} A_m\right)v,\text{ with }A_m := (\Sigma + \kappa_m I)^{-2}\left(\Sigma^2 + \frac{\kappa_m^2\df_2(\kappa_m)}{T-\df_2(\kappa_m)}I\right),
    \end{align}
    where $\kappa_m := \kappa(0,T_m)$ as defined in \eqref{eq:kappa}.
\end{lemma}
\begin{proof}
The first formula follows directly from Lemma \ref{lm:product} with $u$ replaced with $\Sigma u$. For the second formula, we can write
$$
u^\top Q_{n-1}^\top M Q_{n-1} v = u^\top P_0P_1\ldots P_{n-2}P_{n-1}M P_{n-1}P_{n-2}\ldots P_0P_1 v=\widetilde u_{n-1}^\top P_{n-1}M P_{n-1}\widetilde v_{n-1},
$$
where $\widetilde u_{n-1} := P_{n-2}\ldots P_0 u$. So we really only  need to prove the result for $n=1$; the general case will follow by induction and due to multiplicativity. Indeed, defining $A=\Sigma^{1/2}u v^\top \Sigma^{1/2}$, $B=\Sigma^{1/2}M\Sigma^{1/2}$, and $Z_0 = X_0\Sigma^{-1/2}$, we have
\begin{eqnarray}
\begin{split}
    u^\top P_0M P_0 v &= \lim_{t \to 0^+}u^\top X_0^\top (X_0X_0^\top + tI)^{-1}X_0M X_0^\top (X_0X_0^\top + tI)^{-1} X_0 v\\
    &=  \lim_{t \to 0^+} \trace A Z_0 (Z_0\Sigma Z_0^\top + tI)^{-1}Z_0 B Z_0^\top (Z_0Z_0^\top + tI)^{-1}Z_0\\
    &\simeq \trace A (\Sigma + \kappa_0 I)^{-1} B (\Sigma + \kappa_0 I)^{-1} + \kappa_0^2\trace A(\Sigma + \kappa_0 I)^{-2}\cdot \frac{\trace B(\Sigma + \kappa_0 I)^{-2}}{T-\df_2(\kappa_0)}\\
    &= u^\top (\Sigma + \kappa_0 I)^{-1}\Sigma M\Sigma (\Sigma + \kappa_0 I)^{-1}v + \kappa_0^2u^\top \Sigma(\Sigma + \kappa_0 I)^{-2} v\cdot \frac{\trace M\Sigma (\Sigma + \kappa_0)^{-2}}{T-\df_2(\kappa_0)}\\
    &= u^\top \Sigma A_0 v,\text{ for }M=\Sigma,
\end{split}
\end{eqnarray}
where the 3rd line is an application of Proposition 2 of \citet{bach2023highdimensional}.
\end{proof}

\section{Proof of Results for Power-Law Covariance Spectrum}\label{app:B}
\subsection{Proof of Theorem \ref{thm:darkseid}}
From Theorem \ref{thm:generalcov}, 
%Let us pretend that \eqref{eq:RMT} continues to hold even though Assumption \ref{ass:boundedspec} is clearly violated. 
we need to analyze the quantity
\begin{eqnarray}
\begin{aligned}
    \rho &\simeq \frac{\mathrm{df}_2(\kappa(\lambda))}{T_0-d} + \frac{\kappa(\lambda)^2 \trace \left(\Sigma +\kappa(\lambda) I_d\right)^{-2}}{T_0-d}\cdot  \frac{\mathrm{df}_2(\kappa(\lambda))}{T-\mathrm{df}_2(\kappa(\lambda))}.
    \label{eq:rmt}
\end{aligned}
\end{eqnarray}
Now, for small $\lambda$, $\kappa:=\kappa(\lambda)$ is small and one can compute
\begin{eqnarray}
\begin{aligned}
    \mathrm{df}_m(\kappa) &\asymp \sum_i \frac{\lambda_i^m}{(\lambda_i + \kappa)^m} = \kappa^{-m} \sum_i \frac{\lambda_i^m}{(1 + \kappa^{-1}\lambda_i)^m} \asymp \kappa^{-m}\kappa^{(m-1/\beta)} = \kappa^{-1/\beta},
\end{aligned}
\end{eqnarray}
where we have used Lemma \ref{lm:fracture} with $D=\kappa^{-1}$ and $n=m$ in the last step. On the other hand, we can use some of the results of Appendix A (Section 3) of \cite{Cui_2022} to do the following. It can be shown (see aforementioned paper) that
\begin{itemize}
    \item If $\ell > \beta$, then $\kappa \asymp T^{-\beta}$, and so $\mathrm{df}_m(\kappa) \asymp T$ for all $m \ge 1$.
    \item If $\ell < \beta$, then $\kappa \asymp \lambda \asymp T^{-\ell}$, and so $\mathrm{df}_m(\kappa) \asymp T^{\ell/\beta} = o(T)$ for all $m \ge 1$.
\end{itemize}
For $\ell < \beta$, plugging this into \eqref{eq:rmt} gives
\begin{eqnarray}
\begin{aligned}
 \rho &\asymp \frac{T^{\ell/\beta}}{T_0-d} + \frac{d}{T_0-d}\frac{T^{\ell/\beta}}{T-T^{\ell/\beta}} \asymp T_0^{-1}T^{\ell/\beta} + \frac{\phi_0}{1-\phi_0}T^{-(1-\ell/\beta)}\\
 &\asymp \frac{1}{1-\phi_0}\max\left(T/T_0,\phi_0\right) T^{-(1-\ell/\beta)},
 \end{aligned}
\end{eqnarray}
where $\phi_0 := d/T_0$.
Combining our Theorem \ref{thm:generalcov} with \eqref{eq:maincui}, we get the claimed result. \qed

\subsection{Representation of Clean Test Error}
We make a small digression to present the following curiosity: with a slight leap of faith, the main results of \cite{Cui_2022} can be obtained in a few lines from the tools developed in \cite{bach2023highdimensional}, namely Proposition \ref{prop:classicalrige}. This is significant, because the computations in \cite{Cui_2022} were done via methods of statistical physics (replica trick), while \cite{bach2023highdimensional} is based on RMT.

Indeed, for regularization parameter $\lambda \asymp T^{-\ell}$ given in \eqref{eq:reguexpo}, we have $\kappa = \kappa(\lambda) \simeq \lambda$. % \ElvisIssue{Previous statement needs rigorous justification. In CLKZ (2021), it was only sloppily justified.}
% This is because $\kappa-\lambda \simeq \kappa \mathrm{df}_1(\kappa)/T \asymp \kappa^{1-1/\beta}/T$, which is satisfied by $\kappa \simeq \lambda$, since RHS becomes $T^{-(1-1/\beta)\ell}/T=o(1)$.
Thus
\begin{eqnarray}
\kappa \asymp T^{-\ell},\,\mathrm{df}_2(\kappa) \asymp \kappa^{-1/\beta} \asymp T^{\ell / \beta}.
\end{eqnarray}
Now, since $\lambda_i \asymp i^{-\beta}$ (capacity condition) and $(w_0^\top v_i)^2 = c_i^2 \asymp i^{-\delta}$ (source condition), we deduce
\begin{eqnarray}
    \begin{aligned}
\kappa^2 w_0^\top \Sigma(\Sigma + \kappa I)^{-2} w_0 &\asymp w_0^\top \left(\sum_i \frac{\lambda_i}{(\lambda_i + \kappa^{-1}\lambda_i)^2}v_i v_i^\top\right) w_0 = \sum_i \frac{c_i^2\lambda_i}{(\lambda_i + \kappa^{-1}\lambda_i)^2}\\
&= \sum_i \frac{c_i^2 \lambda_i}{(\lambda_i + \kappa^{-1}\lambda_i)^2} \asymp \sum_i \frac{\lambda_i^{1+\delta/\beta}}{(\lambda_i + \kappa^{-1}\lambda_i)^2} \asymp \kappa^{-\gamma} \asymp T^{-\ell \gamma},
\end{aligned}
\label{eq:aha}
\end{eqnarray}
where $\gamma=\min(2,1+\delta/\beta-1/\beta) = \min(2,2r) = 2\underline r$, with $\underline r := \min(r,1)$. The exponent is so because $\delta = 1 + \beta(2r-1)$, and so $\delta / \beta = 1/\beta + 2r-1$ by construction. The estimation of the last sum in \eqref{eq:aha} is thanks to Lemma \ref{lm:fracture} applied with $D=\kappa^{-1}$, $n=1+\delta/\beta$, and $m=2$. Therefore, invoking Proposition \ref{prop:classicalrige} gives
\begin{align}
    Bias &\simeq \frac{\kappa^2 w_0^\top \Sigma(\Sigma + \kappa)^{-2} w_0}{1-\mathrm{df}_2(\kappa)/T} \asymp \frac{T^{\ell \gamma}}{1-T^{-(1-\ell/\beta)}} \asymp T^{-\ell \gamma} = T^{-2\ell \underline r}\\
    Var &\simeq \sigma^2 \frac{\mathrm{df}_2(\kappa)}{T}\cdot \frac{1}{1-\mathrm{df}_2(\kappa)/T} \asymp \sigma^2 \frac{T^{\ell/\beta}}{T}\frac{1}{1-o(1)} \asymp \sigma^2 T^{-(1-\ell/\beta)}.
\end{align}
We deduce the scaling law
\begin{eqnarray}
    E_{test} \simeq Bias + Var \asymp T^{-2\ell \underline r} + \sigma^2 T^{-(1-\ell/\beta)} \asymp \max(\sigma^2,T^{1-2\ell\underline r - \ell/\beta)})T^{-(1-\ell/\beta)},
    \label{eq:maincui}
\end{eqnarray}
which is precisely the main result of \cite{Cui_2022}.

\paragraph{Low-Noise Regime.} In the low noise regime where $\sigma^2 = O(T^{-2\beta \underline r})$, one may take $\ell = \beta$; the variance is then much smaller than the bias, and one has the fast rate
\begin{eqnarray}
E_{test} \asymp T^{-2\beta \underline r}\,.
\end{eqnarray}

\paragraph{High-Noise Regime.} Now, consider the case where $\sigma^2=\Theta(1)$. Setting $2\ell \underline r = 1 - \ell/\beta$ to balance out the bias and variance gives $\ell=\ell_{crit}$, where
\begin{eqnarray}
    \ell_{crit} := \frac{\beta}{2\beta \underline r + 1} \in (0,\beta).
\end{eqnarray}
With this value of the exponent $\ell$, we get the error rate
\begin{eqnarray}
    E_{test} \asymp T^{-2\ell_{crit}\cdot \underline r} = T^{-c},\text{ with }c := \frac{2\beta\underline r}{2\beta \underline r + 1},
\end{eqnarray}
which is precisely the main result of \cite{Cui_2022}, known to be minimax optimal (de Vito \cite{Caponnetto2007OptimalRF}, etc.) !

\section{Auxiliary Lemmas}
 \begin{lemma}
 Let the sequence $(\lambda_k)_{k \ge 1}$ of positive numbers be such that $\lambda_k \asymp k^{-\beta}$ for some constant $\beta > 0$, and let $m,n \ge 0$ with $n\beta > 1$. Then, for $D \gg 1$, it holds that 
 \begin{eqnarray}
     \sum_{k=1}^\infty \frac{\lambda_k^n}{(1+D\lambda_k)^m} \asymp D^{-c}\begin{cases}
         \log D,&\mbox{ if }m=n-1/\beta,\\
         1,&\mbox{ else,}
     \end{cases}
 \end{eqnarray}
 where $c:=\min(m,n-1/\beta) \ge 0$.
 \label{lm:fracture}
 \end{lemma}
 \begin{proof}
    First observe that
    \begin{eqnarray*}
        \begin{aligned}
    \lambda_k^n/(1+D\lambda_k)^m &\asymp \lambda_k^n\min(1,(D\lambda_k)^{-m})\\
    &= \begin{cases}
        \lambda_k^n
        =k^{-n\beta},&\mbox{ if }D\lambda_k < 1, \text{ i.e if }k > D^{1/\beta},\\
        D^{-m}\lambda_k^{-(m-n)}=D^{-m}k^{(m-n)\beta},&\mbox{ else.}
    \end{cases} 
        \end{aligned}
    \end{eqnarray*}
We deduce that
\begin{eqnarray}
\sum_{k=1}^\infty \frac{\lambda_k^n}{(1+D\lambda_k)^m} \asymp D^{-m}\sum_{1 \le k \le D^{1/\beta}}k^{(m-n)\beta} + \sum_{k > D^{1/\beta}}k^{-n\beta}.
\label{eq:fracture}
\end{eqnarray}
By comparing with the corresponding integral, one can write the first sum in \eqref{eq:fracture} as
\begin{eqnarray*}
    \begin{aligned}
D^{-m}\sum_{1 \le k \le D^{1/\beta}}k^{(m-n)\beta} &\asymp D^{-m}\int_1^{D^{1/\beta}}u^{(m-n)\beta}\mathrm{d}u\\
&\asymp D^{-m}
\begin{cases}
    (D^{1/\beta})^{1+(m-n)\beta}=D^{-(n-1/\beta)},&\mbox{ if }n - 1/\beta <  m,\\
    \log D,&\mbox{ if }m=n-1/\beta,\\
    1,&\mbox{ else.} 
\end{cases}\\
&=
\begin{cases}
    D^{-(n-1/\beta)},&\mbox{ if }n - 1/\beta <  m,\\
    D^{-m}\log D,&\mbox{ if }m=n-1/\beta,\\
    D^{-m},&\mbox{ else.}  
\end{cases}\\
&= D^{-c}
\begin{cases}
    \log D,&\mbox{ if }m=n-1/\beta,\\
    1,&\mbox{ else,}  
\end{cases}
    \end{aligned}
\end{eqnarray*}
where $c \ge 0$ is as given in the lemma.

Analogously, one can write the second sum in \eqref{eq:fracture} as
\begin{eqnarray*}
    \begin{aligned}
\sum_{k > D^{1/\beta}}k^{-n\beta} \asymp \int_{D^{1/\beta}}^\infty u^{-n\beta}\mathrm{d}u \asymp (D^{1/\beta})^{1-n\beta} = D^{-(n-1/\beta)},
    \end{aligned}
\end{eqnarray*}
and the result follows upon putting things together.
 \end{proof}

 \begin{lemma}
For $\kappa = \kappa(\lambda,T)$ defined as in \eqref{eq:kappa}, it holds that
    \begin{eqnarray}
        \frac{\partial \kappa}{\partial \lambda} = \frac{1}{1-\df_2(\kappa)/T} \ge 1.   
    \end{eqnarray}
    \label{lm:bach}
\end{lemma}
Thus, perhaps more conveniently, this lemma allows us to rewrite
\begin{align}
    Bias &= w_0^\top \Sigma(\Sigma + \kappa I)^{-2} w_0  \frac{\partial \kappa}{\partial \lambda}\label{eq:omni},\\
    Var &= \sigma^2 \frac{\df_2(\kappa)}{T}\frac{\partial\kappa}{\partial\lambda}.
\end{align}
The RHS of \eqref{eq:omni} is usually referred to as the omniscient risk \citet{Hastie2019Surprises,cui2021generalization,WeiMoreThanToy}.
\begin{proof}[Proof of Lemma \ref{lm:bach}]
    By definition of $\kappa$, we know that
    $$
    \kappa - \lambda = \kappa \df_1(\kappa)/T = \kappa \trace \Sigma(\Sigma + \kappa I)^{-1}/T.
    $$
    Differentiating w.r.t. $\lambda$ gives
    $$
    \kappa' - 1= \kappa'(\trace\Sigma(\Sigma + \kappa I)^{-1} - \kappa\trace\Sigma(\Sigma + \kappa)^{-2})/T =  \kappa'\trace\Sigma^2(\Sigma + \kappa I)^{-2}/T = \kappa'\df_{2}(\kappa)/T,
    $$
    and the result follows upon rearranging.
    Note that we have used the identity
    $$
    I-\kappa(\Sigma + \kappa I)^{-1} = \Sigma(\Sigma + \kappa I)^{-1},
    $$
    to rewrite $\Sigma(\Sigma + \kappa I)^{-1} - \kappa\Sigma(\Sigma + \kappa I)^{-2}= \Sigma^2(\Sigma + \kappa I)^{-2}$.
\end{proof}

%\clearpage

%\input{compare-with-montanari.tex}

\end{document}